\documentclass[10pt,twocolumn,letterpaper]{article}

\usepackage{cvpr}              
\usepackage{xcolor}
\usepackage{colortbl}
\usepackage[accsupp]{axessibility}
\usepackage{listings}
\usepackage{algorithm}
\usepackage{multirow}
\usepackage{lipsum}
\usepackage{amsmath}
\usepackage{amsthm}
\usepackage{amssymb}
\usepackage{graphicx}
\usepackage{floatflt}
\usepackage{wrapfig,lipsum,booktabs}
\usepackage{graphicx}
\usepackage{times}
\usepackage{epsfig}
\usepackage{epstopdf}
\usepackage{array}
\usepackage{pifont}
\usepackage{sidecap}
\usepackage{tablefootnote}
\usepackage{threeparttable}

\definecolor{cvprblue}{rgb}{0.21,0.49,0.74}
\usepackage[pagebackref,breaklinks,colorlinks,allcolors=cvprblue]{hyperref}

\def\paperID{10877} 
\def\confName{ICCV}
\def\confYear{2025}

\title{One Last Attention for Your Vision-Language Model}

\usepackage[symbol]{footmisc}

\author{%
  Liang Chen{\footnotemark[1] \footnotemark[2]}\\
  MBZUAI\\
  {\tt\small liangchen527@gmail.com}
  \and
  Ghazi Shazan Ahmad\footnotemark[1]\\
  MBZUAI\\
  {\tt\small ghazi.ahmad@mbzuai.ac.ae}
  \and
  Tianjun Yao\\
  MBZUAI\\
  {\tt\small tianjun.yao@mbzuai.ac.ae}
  \and
  Lingqiao Liu\\
  The University of Adelaide\\
  {\tt\small lingqiao.liu@adelaide.edu.au}
  \and
  Zhiqiang Shen\footnotemark[3]\\
  MBZUAI\\
  {\tt\small zhiqiang.shen@mbzuai.ac.ae}
}

\begin{document}
\maketitle
\begingroup
  \renewcommand{\thefootnote}{\fnsymbol{footnote}}%
  \footnotetext{$^{*}$Equal technical contribution. $^{\dagger}$Project lead. $^{\ddagger}$Correspondence.}%
\endgroup
\newcommand{\tableCellHeight}{1}
\newcommand{\tabstyle}[1]{
  \setlength{\tabcolsep}{#1}
  \renewcommand{\arraystretch}{\tableCellHeight}
  \centering
  \small
}

\newcolumntype{L}[1]{>{\raggedright\arraybackslash}p{#1}}
\newcolumntype{C}[1]{>{\centering\arraybackslash}p{#1}}
\newcolumntype{R}[1]{>{\raggedleft\arraybackslash}p{#1}}
\newcommand{\romannum}[1]{\romannumeral #1} 
\newcommand{\rotbox}[1]{\rotatebox{70}{#1}}

\newcommand{\hgreen}[1]{\textcolor{ForestGreen}{\textbf{#1}}} 
\newcommand{\hblue}[1]{\textcolor{NavyBlue}{\textbf{#1}}} 
\newcommand{\cavan}[1]{{\color{blue}(cavan: {#1})}} 
\newcommand{\ky}[1]{{\color{red}(ky: {#1})}} 
\definecolor{tabhighlight}{HTML}{e5e5e5}
\definecolor{citecolor}{HTML}{0071bc}
\newcommand{\cmark}{\ding{51}}%
\newcommand{\xmark}{\ding{55}}%

\newtheorem{theorem}{Theorem}[section]
\newtheorem{lemma}[theorem]{Lemma}
\renewcommand{\thefootnote}{\arabic{footnote}}
\begin{abstract}
%
%
Pretrained vision-language models (VLMs), such as CLIP, achieve remarkable zero-shot performance, yet their downstream potential hinges on effective fine-tuning. Most adaptation methods typically focus on refining representation from separate modalities (text or vision) but neglect the critical role of their fused representations in the decision-making process, \emph{\ie} rational matrix that drives the final prediction \cite{chen2023domain}. 
To bridge the gap, we propose a simple yet effective \textbf{R}ational \textbf{Ada}ptaion ({RAda}) to explicitly exploit the final fused representation during fine-tuning. RAda employs a learned mask, obtained from a lightweight attention layer attached at the end of a VLM, to dynamically calibrate the contribution of each element in the rational matrix, enabling targeted adjustments to the final cross-modal interactions without incurring costly modifications to intermediate features.
Experiments in different settings (\emph{\ie} updating, or freezing pretrained encoders in adaptation, and test-time training that can only access the unlabeled test data) show that RAda serves as a versatile fine-tuning technique, improving the baseline with minimal code and performing comparably against current arts in most settings. The full source code can be found at \href{https://github.com/khufia/RAda/tree/main}{github.com/khufia/RAda}.
\end{abstract}    
\section{Introduction}
\label{sec:intro}

Recent foundation models trained on multiple modalities (\eg vision, language, audio) have demonstrated exceptional generalization across diverse tasks. Among these, vision-language models (VLMs) like CLIP~\cite{radford2021learning} and ALIGN~\cite{jia2021scaling}, pretrained on large-scale image-text pairs, achieve remarkable zero-shot classification by aligning input images with text prompts. This alignment is measured by the similarity between image and text representations, with the closest match determining the prediction.
Thanks to the rich supervision provided by the diverse data pairs, such ``zero-shot" classifiers can reason about open-vocabulary visual concepts and obtain impressive robustness to many distribution shifts. 
Nevertheless, in many occasions, it is still beneficial that pretrained VLMs can be adapted to the given data distribution through fine-tuning.  

\begin{figure}
    \centering
    \includegraphics[width=0.95\linewidth]{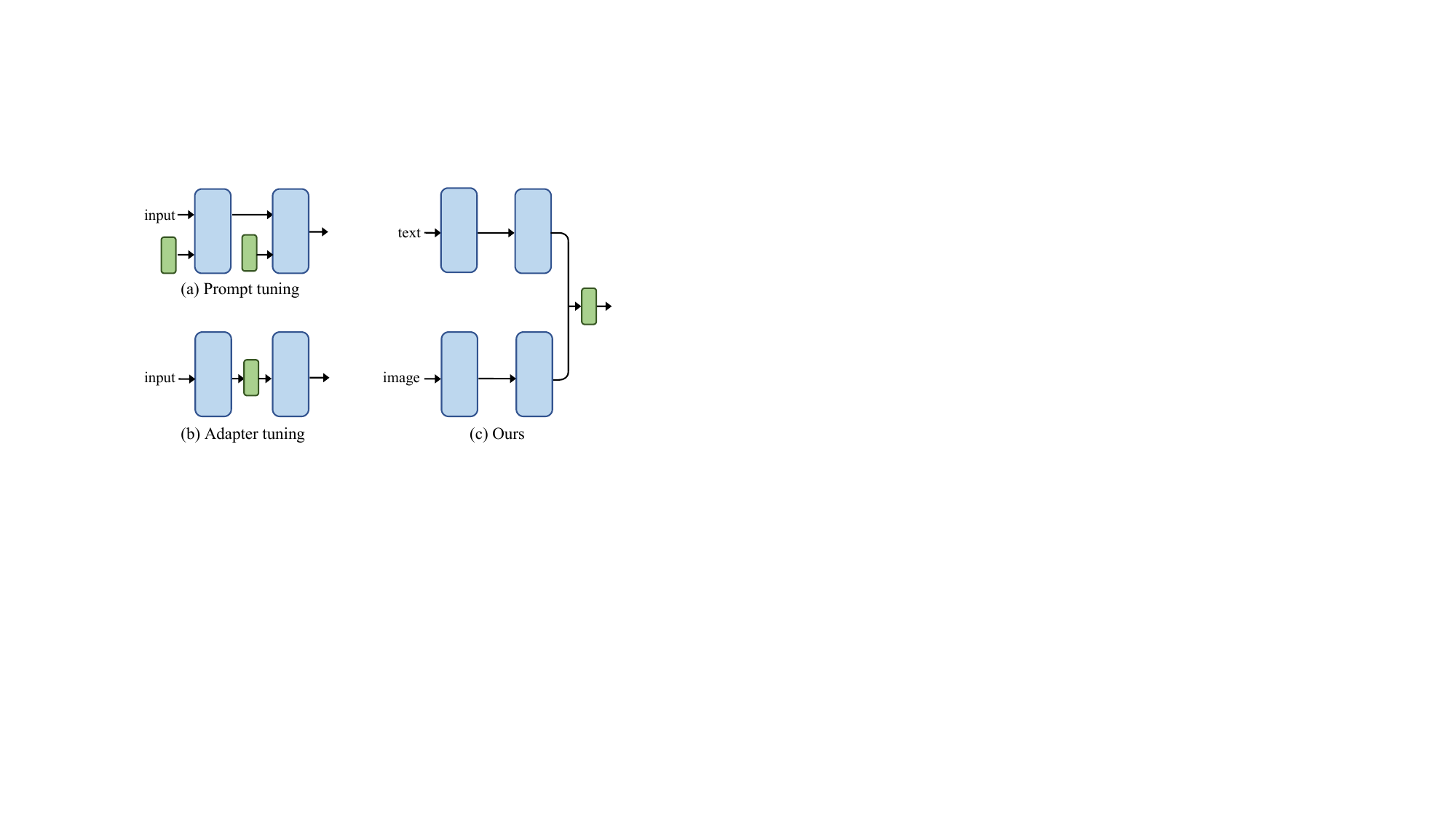}
    \vspace{-0.3 cm}
    \caption{Comparisons between fine-tuning ideas for VLMs that incorporate new parameters. Blue and green blocks denote fixed encoders and learnable parameters. Unlike previous arts that focus on intermediate features from separate modalities, we aim at the fused representations in the final decision-making process, adjusting the corresponding rational matrix~\cite{chen2023domain} to achieve adaptation.}
    \label{fig:tease}
    \vspace{-0.5 cm}
\end{figure}

There are several different fine-tuning strategies applicable for VLMs when the source data is available.
Besides fine-tuning within the standard practice of transfer learning, which updates all parameters~\cite{goyal2023finetune,kumar2022fine} or only the classification head, some recent methods suggest altering the original text or visual representations by tuning newly introduced parameters. Ideas are achieved by either incorporating learnable prompts as an addition to alter the original input embeddings~\cite{zhou2022conditional,zhou2022learning,jia2022visual,shu2022test,zang2022unified} (\ie prompt tuning, as seen in Figure~\ref{fig:tease} (a)), or inserting lightweight learnable parameters within the encoders to modify their outputs~\cite{chen2022vision,gao2024clip,zhang2021tip,ma2021simple,eichenberg2021magma,sung2022vl} (\ie adapter tuning, as seen in Figure~\ref{fig:tease} (b)).
Despite their distinct settings, most methods treat the text or visual features in isolation, overlooking a fundamental aspect of the VLM decision-making: the two features are not independently meaningful, as the final prediction emerges from the interaction of both modalities. This motivates the needs to explicitly emphasize the fused representation in fine-tuning.

A recent art~\cite{khattak2023maple} attempts to address this issue by fostering mutual synergy between the two modalities on the basis of combining vision~\cite{jia2022visual} and text~\cite{zhou2022learning} prompt tunings. While the art shows leading performance in specific scenarios, it incurs relatively large computational overhead due to its requirement for both encoders to actively participate in the intermediate prompt updating steps. More critically, the reliance on both encoders will restrict its applicability when confronting non-ViT-based~\cite{dosovitskiy2020image} architectures or in a standard transfer learning paradigm~\cite{kumar2022fine}, diminishing its practical utility in diverse fine-tuning settings.

To address these challenges, we suggest leveraging fused information from the final decision-making process to achieve lightweight and encoder-agnostic adaptation. 
Unfortunately, in conventional VLMs like CLIP, fused representation at the final stage will be buried within the similarity calculation, rendering it difficult to be accessed directly.
To explicitly surface the buried information, we extend the concept of rational matrix~\cite{chen2023domain}, which is originally developed for the classical classifying system, to VLMs. 
In its classical form, the rational matrix is defined as the entry-wise product between the feature vector and classifier weights, aiming to describe associations between each feature element and classifier weights during prediction.
Extended to VLMs, this concept depicts the fine-grained interaction between visual and textual representations, where the text encoder’s output (acting as classifier weights) interacts with the image encoder’s output (the feature vector) to form a fused representation that governs the final prediction, making it inherently suited for capturing fused information at the VLM decision-making stage.
Through a rational matrix adaptation (see Figure~\ref{fig:tease} (c)), we can thus learn complementary information from a cross-modal perspective, resulting in a more holistic understanding of the data compared to ideas that consider the different modalities in isolation (\emph{empirical and theoretical supports are provided in Sec.~\ref{sec ablation} and the supplementary material, respectively}).

We adopt a streamlined implementation for RAda. Specifically, we propose using a mask, obtained via a lightweight attention layer at the end, to dynamically adjust the contribution of each rational elements for 
achieving adaptation\footnote{In the original CLIP, contributions are all 1 for different elements.}. 
In our design, a multi-query setting is employed for the attention layer to fully exploit all available information, where multiple queries (\ie, image features, text features, and the rational matrix) are applied to the same key and value (\ie the rational matrix), with outputs averaged across all queries. 
The initial prediction can then be adapted by taking the entry-wise product between the original rational matrix and the learned mask. Our method is easy to implement, requiring minimal code upon the baseline CLIP (see Algorithm~\ref{alg 1}).
Unlike~\cite{khattak2023maple}, RAda does not modify intermediate features or require encoder participation, which not only reduces computational overhead but also ensures broader applicability in various fine-tuning settings (detailed comparisons with~\cite{khattak2023maple} are provided in Sec.~\ref{sec maple}).

%
To comprehensively evaluate RAda, besides the settings with source data (\ie (1) full fine-tuning (FFT) that updates all parameters; (2) efficient fine-tuning (EFT) that updates only the rational adapter with frozen encoders), we also evaluate it in a setting when source data is unavailable (\ie (3) test-time training (TTT) that tunes the rational adapter with unlabeled test data).
We observe that RAda exhibits consistent improvements over the baseline in all three settings and perform favorably against existing arts in most scenarios, demonstrating its versatileness and effectiveness in the VLM fine-tuning task. Meanwhile, we also show that existing fine-tuning ideas and RAda are not mutually exclusive, rather, their integration shows further improvements.

Main contributions of this work are three-fold:
\begin{itemize}
    \item We propose \emph{rational adaptation}, a novel method that extends the concept of rational matrix from the classical classifying system to VLMs, to favorably focus on the fused text and visual representations at the final decision-making process during fine-tuning. 
    \item We offer a simple and lightweight implementation for RAda, fulfilled by attaching a single attention layer at the end to learn contributions for different rational elements. This design can be seamlessly integrated into most training pipelines with minimal code.
    \item We conduct extensive experiments in three mainstream fine-tuning settings to evaluate RAda. We observe that RAda is a versatile fine-tuning idea that can consistently benefit the baseline, and it can obtain comparable performance against existing arts in most settings.
\end{itemize}

\section{Related Work}
\label{sec:formatting}

\textbf{Vision-Language models.}
Previous studies have demonstrated the effectiveness of using text supervision for various vision tasks~\cite{radford2021learning,jia2021scaling,zhang2022contrastive,li2021supervision,yao2021filip,yuan2021florence,zhai2022lit}. 
%
%
Attributed to the large-scale training data from the web, current VLMs can achieve astonishing results on a wide spectrum of vision tasks without any fine-tuning~\cite{radford2021learning,jia2021scaling}. 
%
Similar to prevailing fine-tuning methods~\cite{goyal2023finetune,zhou2022conditional,gao2024clip,shu2022test}, our implementation builds on the pretrained CLIP, aiming to enhance its performance across various fine-tuning strategies.

\noindent\textbf{Fine-tuning strategies.}
We briefly review some of the fine-tuning arts by categorizing them into the following types. 

The first common paradigm for fine-tuning is FFT. 
While it can effectively adapt models to a new distribution, the overfitting problem remains the primary concern in the literature, causing compromised robustness in diverse tasks~\cite{kumar2022fine,chen2023improved,chen2024lfme,chen2024causal}. Several ideas have been proposed lately to mitigate the issue, including LP-FT~\cite{kumar2022fine} that conducts linear probing (LP) and fine tuning (FT) in a sequential manner, and weight ensemble~\cite{wortsman2022robust} that combines weights of both the fine-tuned and pretrained models. We extend the sequential training strategy in LP-FT for RAda, where we first train the rational adapter as initialization and then finetune all parameters. Our experiments indicate that RAda can better help ease the overfitting problem than previous arts while maintaining comparable effectiveness in the training distribution.  
%

The second idea that has been widely explored is EFT, which fixes encoders during updating. The key in EFT is to introduce new learnable parameters for adaptation. Inspired by CoOp~\cite{zhou2022learning}, a wide range of nascent studies use learnable prompts in the vision or text encoders as additional inputs~\cite{zhou2022conditional,jia2022visual,shu2022test,zang2022unified,khattak2023maple}. Adapter tuning also shows its effectiveness in fine-tuning~\cite{chen2022vision,gao2024clip,zhang2021tip,ma2021simple,eichenberg2021magma,sung2022vl,shimomoto2023towards}, aiming to modify the original model by inserting layers to act on the input representations, unlike prompt tuning which modifies the inputs themselves.
RAda can be extended to EFT by tuning only the rational adapter. Unlike previous ideas, it specifically focuses on the fused representations, aiming for a more holistic understanding of the new data.

The last is the emerging TTT that updates the model in the test phase. The idea is to leverage a self-supervised task to update the model with test samples on the fly~\cite{sun2020test,wang2020tent,liu2021ttt++}. 
The concept has been explored in some recent studies to fully uncover the zero-shot potential of CLIP~\cite{shu2022test,abdul2024align,zanella2024test,zhao2024testtime}, with most approaches focusing on updating newly introduced parameters (\ie prompts) for CLIP using the combinations of a basic entropy minimization task~\cite{wang2020tent} and other hand-crafted objectiveness, such as distribution alignment~\cite{abdul2024align} or advanced feedback from a larger model~\cite{zhao2024testtime}.
In line with these efforts, we update the rational adapter during test using the same entropy minimization objective.

To the best of our knowledge, this work represents a pioneering effort to explore and adapt a fine-tuning idea that can contribute effectively in all three fine-tuning settings.

\section{Methodology}

\subsection{Preliminary}
Our method is built upon a representative VLM, CLIP, which includes two parallel encoders for mapping the text and visual inputs into feature vectors. We denote the text and visual encoders as $\mathcal{F}_t$ and $\mathcal{F}_v$ and their pretrained parameters as $\theta_t$ and $\theta_v$, respectively.
Considering a $K$-class classification task in the fine-tuning process, the visual representation $\mathbf{f} \in \mathbb{R}^D$ for a given input image $\mathbf{I} \in \mathbb{R}^{C\times H\times W}$ can be simply obtained via $\mathbf{f} = \mathcal{F}_v(\mathbf{I}, \theta_v)$, and the text representations $\mathbf{h} \in \mathbb{R}^{K\times D}$ are decided upon shared prefixed text prompts $\mathbf{p}$ for different predictions. For instance, a commonly-used $\mathbf{p}=$``a photo of a" can be adopted as the prefix prompts, and each text representation $\mathbf{h}_{i} \in \mathbb{R}^D$ is then obtained via $\mathbf{h}_{i} = \mathcal{F}_t([\mathbf{p},\mathcal{Y}_i], \theta_t)$, given $\mathcal{Y}=\{\mathcal{Y}_1,\mathcal{Y}_2,...,\mathcal{Y}_K\}$ the category-specific texts for all $K$ classes, and $[\mathbf{p},\mathcal{Y}_i] =$ ``a photo of a [class]" being a direct concatenate of $\mathbf{p}$ and $\mathcal{Y}_i$.

\begin{figure}
    \centering
    \includegraphics[width=0.97\linewidth]{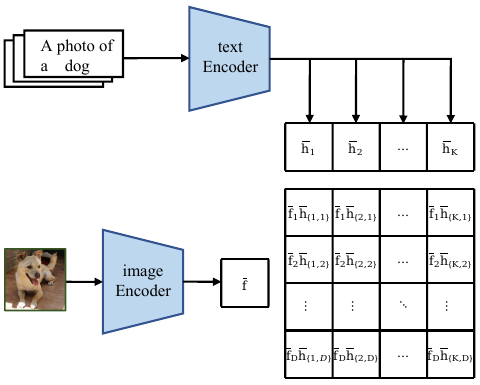}
    \vspace{-0.3cm}
    \caption{Rational matrix~\cite{chen2023domain} in the CLIP decision-making process for a given image, where predictions (\ie logits) are computed by summing each column. It fuses text and visual features and provides a fine-grained characterization of different predictions.}
    \vspace{-0.5 cm}
    \label{fig:rational}
\end{figure}

Similar to that in pretraining CLIP, the objective in fine-tuning is still contrastive learning. Take the updating of $\theta_v$ as an example, the goal is to align the input $\mathbf{f}$ with its corresponding text description $\mathbf{h}_{\ast}$ while away from others, which can be formulated as minimizing the following,
\begin{equation}
    \label{eq obj}
    \mathcal{L}_{main} = -\log \frac{\exp(<\overline{\mathbf{f}},~\overline{\mathbf{h}}_{\ast}>)}{\sum_{i=1}^K \exp(<\overline{\mathbf{f}},~\overline{\mathbf{h}}_{i}>)},
\end{equation}
where $\overline{\mathbf{f}}$ and $\overline{\mathbf{h}}$ are the $l_2$ normalized ${\mathbf{f}}$ and ${\mathbf{h}}$, $<,~>$ returns the inner product for two vectors. In evaluation, the class $k$ that with the largest logit $<\overline{\mathbf{f}},~\overline{\mathbf{h}}_{k}>$ is considered the true prediction. 
Note that in the pretraining process, $\mathcal{L}_{main}$ should be considered with another contrastive case where each text representation would correspond to different visual features, and the alignment is to ensure they are paired with the correct visual targets, same as that in~Eq.~\eqref{eq obj}.

\subsection{Rational Matrix in CLIP}
The default design in Eq.~\eqref{eq obj} cannot characterize fine-grained associations between $\mathbf{f}$ and $\mathbf{h}$ as it leads directly to the coarse final result.
To surface the fused information from the two modalities, we suggest extending the concept of rational matrix~\cite{chen2023domain} to CLIP.
The rational matrix is regarded as a fine-grained characterization of the decision-making process in~\cite{chen2023domain}. The concept is originally introduced within the classical classification system, where a linear classifier is involved for computing logits from the obtained image feature $\mathbf{f}$: given the classifier $\mathbf{W} \in \mathbb{R}^{D\times K}$ and rational matrix $\hat{\mathbf{R}} \in \mathbb{R}^{K\times D}$, the logit value for the $i$-th class is $\mathbf{o}_i = <\mathbf{f},~\mathbf{W}_{\{,i\}}> = \sum_{j=1}^D \mathbf{f}_j \mathbf{W}_{\{j, i\}} = \sum_{j=1}^D \hat{\mathbf{R}}_{\{i,j\}}$.

In CLIP, the text feature $\overline{\mathbf{h}}$ plays the same role as the ``classifier" $\mathbf{W}$ when calculating the similarity (\ie the logits computation process). When studying the decision-making process in CLIP (with the same $\mathbf{f}$), we can correspondingly depict it with a rational matrix $\mathbf{R} \in \mathbb{R}^{K\times D}$ as,
\begin{equation}
\label{eq rationale}
\mathbf{R}^\top = 
\begin{bmatrix}
\overline{\mathbf{f}}_1 \overline{\mathbf{h}}_{\{1, 1\}}  &\overline{\mathbf{f}}_1 \overline{\mathbf{h}}_{\{2, 1\}}  &$\dots$ &\overline{\mathbf{f}}_1 \overline{\mathbf{h}}_{\{K, 1\}} \\
\overline{\mathbf{f}}_2 \overline{\mathbf{h}}_{\{1, 2\}}  &\overline{\mathbf{f}}_2 \overline{\mathbf{h}}_{\{2, 2\}} &$\dots$ &\overline{\mathbf{f}}_2 \overline{\mathbf{h}}_{\{K, 2\}} \\
\vdots &\vdots & \ddots &\vdots \\
\overline{\mathbf{f}}_D \overline{\mathbf{h}}_{\{1, D\}}  &\overline{\mathbf{f}}_D \overline{\mathbf{h}}_{\{2, D\}}  &$\dots$ & \overline{\mathbf{f}}_D \overline{\mathbf{h}}_{\{K, D\}} \\
\end{bmatrix}.
\end{equation}
Similarly, the $i$-th logit in the CLIP result can be represented as $\sum_{j=1}^D \mathbf{R}_{\{i,j\}}$.
An example of obtaining $\mathbf{R}$ for an image is provided in Figure~\ref{fig:rational}. $\mathbf{R}$ encodes the interactions between the visual and textual features in the final decision-making process of CLIP, manifesting the fine-grained fused information at the final stage. Compared to the isolated $\mathbf{f}$ and $\mathbf{h}$, it offers a more holistic understanding of the data by integrating information from both modalities.

By depicting the final-decision making process with the rational matrix, objective in Eq.~\eqref{eq obj} can be rewritten as,
\begin{equation}
    \label{eq objreform}
    \mathcal{L}_{main} \triangleq -\log \frac{\exp(<\textbf{1}_D, \mathbf{R}_{\ast}>)}{\sum_{i=1}^K \exp(<\textbf{1}_D, \mathbf{R}_{i}>)},
\end{equation}
where $\textbf{1}_D$ is a D-dimensional all-one vector.
\emph{Note that reformulating Eq.~\eqref{eq obj} to Eq.~\eqref{eq objreform} does not introduce any additional parameters}. We show in the following that the new formulation in Eq.~\eqref{eq objreform} can provide a new perspective for adaptation in the fine-tuning process with minimum cost.


\subsection{Rational Adaptation in Fine-Tuning CLIP}
Unlike previous methods~\cite{jia2022visual,ma2021simple,zhou2022conditional,gao2024clip} that use various forms to extract complementary text or visual information from the new data while processing them in isolation, we suggest a new fine-tuning idea by adapting the corresponding rational matrix, which can specifically leverage fused representations from different modalities. 
Specifically, we adopt a learned mask $\mathbf{M}$, which is with continuous values and with the same shape as $\mathbf{R}$, to dynamically calibrate contributions of each rational element in making the final predictions. 
Formally, with the adaptive $\mathbf{M}$, the contrastive objective in Eq.~\eqref{eq objreform} is evolved into,
\begin{equation}
    \label{eq envolveobj}
    \mathcal{L}_{adapt} = -\log \frac{\exp(<\textbf{1}_D, (\mathbf{M}\circ\mathbf{R})_{\ast}>)}{\sum_{i=1}^K \exp(<\textbf{1}_D, (\mathbf{M}\circ\mathbf{R})_{i}>)},
\end{equation}
where $\circ$ denotes the Hadamard product.

Eq.~\eqref{eq envolveobj} can be seamlessly extended to the entropy minimization task in test-time fine-tuning~\cite{wang2020tent} as well, where the ground-truth $\mathbf{h}_{\ast}$ is unavailable. By tuning $\mathbf{M}$, we can formulate the main objective in TTT as,
\begin{equation}
\footnotesize
    \label{eq ttt}
    \mathcal{L}_{ttt} = -\sum_{j=1}^K p_j \log p_j,~~\text{s.t.}~p_j=\frac{\exp(<\textbf{1}_D, (\mathbf{M}\circ\mathbf{R})_{j} >)}{\sum_{i=1}^K \exp(<\textbf{1}_D, (\mathbf{M}\circ\mathbf{R})_{i} >)}.
\end{equation}

\begingroup
\setlength{\textfloatsep}{-5pt}
\begin{algorithm}[tb]
   \caption{PyTorch-style pseudocode for RAda in EFT.}
   \label{alg 1}
   
    \definecolor{codeblue}{rgb}{0.25,0.5,0.5}
    \lstset{
    backgroundcolor=\color{white},
  basicstyle=\fontsize{7.5pt}{7.5pt}\ttfamily\selectfont,
  columns=fullflexible,
  breaklines=true,
  captionpos=b,
  commentstyle=\fontsize{7.5pt}{7.5pt}\color{codeblue},
  keywordstyle=\fontsize{7.5pt}{7.5pt},
}
\begin{lstlisting}[language=python]
# CLIP_encoder: Include vision and text encoders
# Attn: Attention layer
# I[BxHxWxC]: Batch of visual inputs 
# T[KxL]: Text inputs in forms of "a photo of a []"

# extract and normalize features of each modality
f,   h   = CLIP_encoder(I, T)  #BxD, KxD
f_n, h_n = l2_norm(f, 1), l2_norm(h, 1)

# compute the rational matrix and the mask
f_e = f_n.unsqueeze(1).repeat(1,K,1)  #BxKxD
h_e = h_n.unsqueeze(0).repeat(B,1,1)  #BxKxD
R   = f_e * h_e 
M   = Attn(query=[f_e, h_e, R], key=R, value=R)  #BxKxD

# obtain adapted results, baseline CLIP is with M=1
logits = torch.sum(M * R, -1)  #BxK

# compute the loss
L_main = cross_entropy_loss(logits, label)
L_reg = mse_loss(M, 1)
\end{lstlisting}
\end{algorithm}
\endgroup

\noindent\textbf{Adaptively Learning $\mathbf{M}$.}
%
To adapt $\mathbf{R}$ w.r.t different inputs, we suggest obtaining $\mathbf{M}$ through a learnable function $\mathcal{F}_m$, with $\theta_m$ denoted as its parameter. 
Considering that the original ``classifier", \ie the text embeddings $\overline{\mathbf{h}}$, does not account for the presence of other to-be-distinguished classes when making a prediction, we suggest incorporating an attention mechanism~\cite{vaswani2017attention} for implementing $\mathcal{F}_m$, with the original $\mathbf{R}$ as the input query, key, and value, enabling interactions between rationale elements from different classes. 
%
%
Meanwhile, to account for the original information embedded in different modalities, we suggest also using the image and text features $\{\overline{\mathbf{f}}$, $\overline{\mathbf{h}}\}$ as additional queries for $\mathcal{F}_m$, helping it to learn more nuanced relationships and dependencies across the different rational elements.
In our implementation, we use different projectors for the different queries $\{\overline{\mathbf{f}}, \overline{\mathbf{h}}, \mathbf{R}\}$. The result is then computed from the average attention weight of these queries,
\begin{equation}
\small
\begin{aligned}
    \label{eq mask}
    \mathbf{M}' &= \mathcal{F}_m(\{\overline{\mathbf{f}}, \overline{\mathbf{h}}, \mathbf{R}\}, \theta_m)\\
    &=(\text{sfm}(\frac{Q_{\overline{\mathbf{f}}} K^\top}{\sqrt{d_K}}) + \text{sfm}(\frac{Q_{\overline{\mathbf{h}}} K^\top}{\sqrt{d_K}}) + \text{sfm}(\frac{Q_\mathbf{R} K^\top}{\sqrt{d_K}}))\frac{V}{3}\\
    &\text{s.t.}~~Q_x =  W_{x}^\top x,~~K = W_k^\top \mathbf{R},~~V = W_v^\top \mathbf{R},
\end{aligned}
\end{equation}
where sfm($\cdot$) denotes the softmax function, $W_{x}$ is the linear projector for the corresponding query $x$; $W_k$ and $W_v$ are the key and value projectors. 
%
%
%
We omit the last projection layer in $\mathcal{F}_m$, which are zero-initialized in all three settings, and we use $\mathbf{M} = \mathbf{M}' + \textbf{1}_{K\times D}$ for implementation. These designs ensure $\mathbf{M}=\textbf{1}_{K\times D}$ in the first updating step, and it does not affect the initial predictions of CLIP.

\noindent\textbf{Regularization for $\mathbf{M}$.}
Although fine-tuning can improve the in-domain (ID) accuracy, it inevitably diminishes the strong zero-shot capabilities of the original CLIP model. As such, we suggest using a regularization for $\mathbf{M}$ to explicitly maintain the zero-shot performance.
Specifically, we introduce a smoothness constraint on the mask, designed to prevent significant deviations from its original setting (\ie ensuring each element in $\mathbf{R}$ contributes equally in the initial decision-making process), which formally gives,
\begin{equation}
    \label{eq reg}
    \mathcal{L}_{reg} = \Vert \mathbf{M} - \textbf{1}_{K\times D} \Vert^2.
\end{equation}
Notably, $\mathcal{L}_{\text{reg}}$ is applied only when the pretrained encoders are frozen during fine-tuning. Since in other situations, $\mathbf{R}$ will be different from its initial state, and the pretrained information cannot be retained even with $\mathbf{M}=\textbf{1}_{K\times D}$.

\noindent\textbf{Overall Algorithm.}
Objectives in three different fine-tuning settings are summarized as, 1) in FFT, we minimize $\mathcal{L}_{adapt}$ regarding all learnable parameters; 2) in EFT, we minimize $\mathcal{L}_{adapt} + \mathcal{L}_{reg}$ w.r.t $\theta_m$; 3) in TTT, the objective is to minimize $\mathcal{L}_{ttt} + \mathcal{L}_{reg}$ regarding $\theta_m$.
Pseudo code for a representative fine-tuning setting EFT is illustrated in Algorithm~\ref{alg 1}. As seen, the proposed method is extremely simple, as it only adds a few lines on the baseline CLIP.

\begin{table}[t]
\caption{\textbf{Evaluations in the FFT setting}. Results with $\dag$ are reevaluated in our device, others are from FLYP~\cite{goyal2023finetune}~\tablefootnote{Results are diverse in the ObjectNet dataset for CLIP is mainly because FLYP uses a different version of the original dataset, as also noted in their official project~\url{https://github.com/locuslab/FLYP}.}.}
\vspace{-0.3 cm}
\centering
\scalebox{0.78}{
\begin{tabular}{@{}cccccccc@{}}
\toprule
  & \multicolumn{7}{c}{Training in Imagenet}  \\ 
  \cmidrule(l){2-8} 
\multirow{-2}{*}{Methods} & ID   & Im-V2  & Im-R &Im-A &Im-S &Object  & OOD Avg.  \\
\midrule 
CLIP & \multicolumn{1}{c|}{68.3} &61.9 &77.7 &50.0 &48.3 &55.4 & {58.7} \\
LP & \multicolumn{1}{c|}{79.9} &69.8 &70.8 &46.4 &46.9 &52.1 &{57.2} \\ 
FT & \multicolumn{1}{c|}{81.3} &70.9 &65.6 &36.7 &46.3 &49.6 &{53.8}\\ 
L2-SP & \multicolumn{1}{c|}{81.7} &71.8 &70.0 &42.5 &48.5 &56.2 &57.8\\
LP-FT & \multicolumn{1}{c|}{81.7} &72.1 &73.5 &47.6 &50.3 &58.2 &60.3\\
FLYP & \multicolumn{1}{c|}{82.6} &73.0 &71.4 &48.1 &49.6 &58.7 &60.2\\
\hline
CLIP$\dag$ & \multicolumn{1}{c|}{68.3} &61.9 &\textbf{77.7} &49.9 &48.2 &54.2 & {58.4} \\
FLYP$^\dag$ & \multicolumn{1}{c|}{\textbf{82.6}} &\textbf{72.6} &71.8 &48.5 &49.8 &54.6 &59.5\\
\rowcolor{tabhighlight}
RAda & \multicolumn{1}{c|}{78.1} &68.3 &72.8 &47.3 &46.9 &53.8 & {57.8} \\
\rowcolor{tabhighlight}
RAda-FT & \multicolumn{1}{c|}{81.4} &71.9 &75.5 &\textbf{51.7} &\textbf{50.4} &\textbf{56.8} & \textbf{61.3} \\
\bottomrule
\end{tabular}
}
\label{tab fft}
\vspace{-0.5 cm}
\end{table}

\definecolor{darkgray}{rgb}{0.6, 0.6, 0.6}

\begin{table*}[t]
\tabstyle{6pt}
    \caption{{Comparison with EFT methods in the base-to-new setting}. All methods are learned from the base classes with 16 shots. RAda + Prompt indicates combining RAda with  vision and text prompt tunings~\cite{jia2022visual, zhou2022learning}, which are also utilized in~\cite{khattak2023maple}. RAda + Adapter indicates combining RAda with adapters within both encoders (the foundation skill also adopted in~\cite{yang2024mma}). 'HM' denotes harmonic mean.}
    \vspace{-0.3 cm}
    \label{tab b2n}
\resizebox{0.95\textwidth}{!}{\begin{tabular}{c|ccc|ccc|ccc|ccc}
\hline
 \multirow{2}{*}{Methods} &
  \multicolumn{3}{c|}{{Average}} &
  \multicolumn{3}{c|}{{ImageNet}} &
  \multicolumn{3}{c|}{{Caltech101}} &
  \multicolumn{3}{c}{{OxfordPets}} \\
&
  Base & New & HM & Base & New & HM & Base & New & HM & Base & New & HM \\ \hline
CLIP~\cite{radford2021learning} & \textcolor{darkgray}{69.34} & \textcolor{darkgray}{74.22} & 71.70 & \textcolor{darkgray}{72.43} & \textcolor{darkgray}{68.14} & 70.22 
& \textcolor{darkgray}{96.84} & \textcolor{darkgray}{94.00} & 95.40 & \textcolor{darkgray}{91.17} & \textcolor{darkgray}{97.26} & 94.12 \\
CoOp~\cite{zhou2022learning} &\textcolor{darkgray}{82.69} & \textcolor{darkgray}{63.22} & 71.66 & \textcolor{darkgray}{76.47} & \textcolor{darkgray}{67.88} & 71.92 & \textcolor{darkgray}{98.00} & \textcolor{darkgray}{89.81} & 93.73 & \textcolor{darkgray}{93.67} & \textcolor{darkgray}{95.29} & 94.47 \\
CoCoOp~\cite{zhou2022conditional} &\textcolor{darkgray}{80.47} & \textcolor{darkgray}{71.69} & 75.83 &\textcolor{darkgray}{75.98} & \textcolor{darkgray}{70.43} & {73.10} & \textcolor{darkgray}{97.96} & \textcolor{darkgray}{93.81} & {95.84} & \textcolor{darkgray}{95.20} & \textcolor{darkgray}{97.69} & {96.43} \\
ProGrad~\cite{zhu2023prompt} &\textcolor{darkgray}{82.48} &\textcolor{darkgray}{70.75} &76.16 &\textcolor{darkgray}{77.02} &\textcolor{darkgray}{66.66} &71.46 &\textcolor{darkgray}{98.02} &\textcolor{darkgray}{93.89} &95.91  &\textcolor{darkgray}{95.07} &\textcolor{darkgray}{97.63} &96.33\\
CLIP-Adapter~\cite{gao2024clip} &\textcolor{darkgray}{80.83} & \textcolor{darkgray}{72.93} & 76.67 &\textcolor{darkgray}{75.78} & \textcolor{darkgray}{67.60} & 71.45
& \textcolor{darkgray}{98.32} & \textcolor{darkgray}{93.56} & 95.88 & \textcolor{darkgray}{93.73} & \textcolor{darkgray}{95.97} & 94.84 \\
KgCoOp~\cite{yao2023visual} & \textcolor{darkgray}{80.73} & \textcolor{darkgray}{73.60} & 77.00 & \textcolor{darkgray}{75.83} & \textcolor{darkgray}{69.96} & 72.78 & \textcolor{darkgray}{97.72} & \textcolor{darkgray}{94.39} & 96.03 & \textcolor{darkgray}{94.65} & \textcolor{darkgray}{97.76} & 96.18 \\
MaPLe~\cite{khattak2023maple} & \textcolor{darkgray}{82.28} & \textcolor{darkgray}{75.14} & 78.55 & \textcolor{darkgray}{76.66} & \textcolor{darkgray}{70.54} & 73.47 & \textcolor{darkgray}{97.74} & \textcolor{darkgray}{94.36} & 96.02 & \textcolor{darkgray}{95.43} & \textcolor{darkgray}{97.76} & 96.58 \\
DePT~\cite{zhang2024dept} + MaPLe & \textcolor{darkgray}{84.85} & \textcolor{darkgray}{74.82} & 79.52 & \textcolor{darkgray}{77.87} & \textcolor{darkgray}{70.23} & 73.85 & \textcolor{darkgray}{98.53} & \textcolor{darkgray}{95.03} & \textbf{96.75} & \textcolor{darkgray}{95.03} & \textcolor{darkgray}{97.83} & 96.41\\ 
MMA~\cite{yang2024mma} & \textcolor{darkgray}{83.20} & \textcolor{darkgray}{76.80} & 79.87  
& \textcolor{darkgray}{77.31} & \textcolor{darkgray}{71.00} & \textbf{74.02}  
& \textcolor{darkgray}{98.40} & \textcolor{darkgray}{94.00} & 96.15  
& \textcolor{darkgray}{95.40} & \textcolor{darkgray}{98.07} & 96.72 \\  
\rowcolor{tabhighlight}
RAda &\textcolor{darkgray}{82.16} & \textcolor{darkgray}{74.14} & {77.94} & \textcolor{darkgray}{75.50} &\textcolor{darkgray}{68.41} &71.78 &\textcolor{darkgray}{98.39} &\textcolor{darkgray}{94.32} &96.31 &\textcolor{darkgray}{94.31} &\textcolor{darkgray}{96.03} &95.16\\
\rowcolor{tabhighlight}
RAda + Prompt &\textcolor{darkgray}{84.18} & \textcolor{darkgray}{75.61} & {79.67} & \textcolor{darkgray}{77.11} &\textcolor{darkgray}{68.29} & 72.44 &\textcolor{darkgray}{98.01} &\textcolor{darkgray}{94.65} &96.31 &\textcolor{darkgray}{96.16} &\textcolor{darkgray}{97.87} &\textbf{97.01}\\
\rowcolor{tabhighlight}  
RAda + Adapter & \textcolor{darkgray}{84.32} & \textcolor{darkgray}{76.25} & \textbf{80.08}
& \textcolor{darkgray}{77.96} & \textcolor{darkgray}{70.23} & 73.89  
& \textcolor{darkgray}{98.06} & \textcolor{darkgray}{93.56} & 95.76  
& \textcolor{darkgray}{95.43} & \textcolor{darkgray}{97.99} & 96.69 

 \\ \hline
 \end{tabular}}
 \resizebox{0.95\textwidth}{!}{\begin{tabular}{c|ccc|ccc|ccc|ccc}
\hline
\multirow{2}{*}{Methods} &
  \multicolumn{3}{c|}{{StanfordCars}} &
  \multicolumn{3}{c|}{{Flowers102}} &
  \multicolumn{3}{c|}{{Food101}} &
  \multicolumn{3}{c}{{FGVCAircraft}} \\
 &
  Base & New & HM & Base & New & HM & Base & New & HM & Base & New & HM \\ \hline
CLIP~\cite{radford2021learning} & \textcolor{darkgray}{63.37} & \textcolor{darkgray}{74.89} & 68.65 & \textcolor{darkgray}{72.08} & \textcolor{darkgray}{77.80} & 74.83 & \textcolor{darkgray}{90.10} & \textcolor{darkgray}{91.22} & 90.66 & \textcolor{darkgray}{27.19} & \textcolor{darkgray}{36.29} & 31.09 \\
CoOp~\cite{zhou2022learning} & \textcolor{darkgray}{78.12} & \textcolor{darkgray}{60.40} & 68.13 & \textcolor{darkgray}{97.60} & \textcolor{darkgray}{59.67} & 74.06 & \textcolor{darkgray}{88.33} & \textcolor{darkgray}{82.26} & 85.19 & \textcolor{darkgray}{40.44} & \textcolor{darkgray}{22.30} & 28.75\\
CoCoOp~\cite{zhou2022conditional} & \textcolor{darkgray}{70.49} & \textcolor{darkgray}{73.59} & 72.01 & \textcolor{darkgray}{94.87} & \textcolor{darkgray}{71.75} & 81.71 & \textcolor{darkgray}{90.70} & \textcolor{darkgray}{91.29} & 90.99 & \textcolor{darkgray}{33.41} & \textcolor{darkgray}{23.71} & 27.74 \\
ProGrad~\cite{zhu2023prompt} & \textcolor{darkgray}{77.68} & \textcolor{darkgray}{68.63} & 72.88 & \textcolor{darkgray}{95.54} & \textcolor{darkgray}{71.87} & 82.03 & \textcolor{darkgray}{90.37} & \textcolor{darkgray}{89.59} & 89.98 & \textcolor{darkgray}{40.54} & \textcolor{darkgray}{27.57} & 32.82\\
CLIP-Adapter~\cite{gao2024clip} & \textcolor{darkgray}{73.64} & \textcolor{darkgray}{71.50} & 72.55 & \textcolor{darkgray}{96.77} & \textcolor{darkgray}{71.56} & 82.28 & \textcolor{darkgray}{90.16} & \textcolor{darkgray}{90.96} & 90.56 & \textcolor{darkgray}{35.65} & \textcolor{darkgray}{32.27} & 33.87 \\
KgCoOp~\cite{yao2023visual} & \textcolor{darkgray}{71.76} & \textcolor{darkgray}{75.04} & 73.36 & \textcolor{darkgray}{95.00} & \textcolor{darkgray}{74.73} & 83.65 & \textcolor{darkgray}{90.50} & \textcolor{darkgray}{91.70} & 91.09 & \textcolor{darkgray}{36.21} & \textcolor{darkgray}{33.55} & 34.83\\
MaPLe~\cite{khattak2023maple} & \textcolor{darkgray}{72.94} & \textcolor{darkgray}{74.00} & 73.47 & \textcolor{darkgray}{95.92} & \textcolor{darkgray}{72.46} & 82.56 & \textcolor{darkgray}{90.71} & \textcolor{darkgray}{92.05} & \textbf{91.38} & \textcolor{darkgray}{37.44} & \textcolor{darkgray}{35.61} & 36.50 \\
DePT~\cite{zhang2024dept} + MaPLe & \textcolor{darkgray}{80.93} & \textcolor{darkgray}{71.73} & 76.06 & \textcolor{darkgray}{98.03} & \textcolor{darkgray}{73.17} & 83.79 & \textcolor{darkgray}{90.33} & \textcolor{darkgray}{91.53} & 90.93 & \textcolor{darkgray}{44.53} & \textcolor{darkgray}{32.80} & 37.78\\ 
MMA~\cite{yang2024mma} & \textcolor{darkgray}{78.50} & \textcolor{darkgray}{73.10} & 75.70  
& \textcolor{darkgray}{97.77} & \textcolor{darkgray}{75.93} & \textbf{85.48}  
& \textcolor{darkgray}{90.13} & \textcolor{darkgray}{91.30} & 90.71  
& \textcolor{darkgray}{40.57} & \textcolor{darkgray}{36.33} & 38.33 \\  
\rowcolor{tabhighlight}
RAda & \textcolor{darkgray}{76.29} & \textcolor{darkgray}{73.73} & 74.99 & \textcolor{darkgray}{95.63} & \textcolor{darkgray}{72.77} & 82.65 & \textcolor{darkgray}{90.01} & \textcolor{darkgray}{90.55} & 90.28 & \textcolor{darkgray}{38.90} & \textcolor{darkgray}{33.65} & 36.09\\
\rowcolor{tabhighlight}
RAda + Prompt & \textcolor{darkgray}{79.15} & \textcolor{darkgray}{73.93} & \textbf{76.45} & \textcolor{darkgray}{97.25} & \textcolor{darkgray}{69.86} & 81.31 & \textcolor{darkgray}{90.33} & \textcolor{darkgray}{91.17} & 90.75 & \textcolor{darkgray}{41.83} & \textcolor{darkgray}{35.03} & 38.13\\
\rowcolor{tabhighlight}  
RAda + Adapter & \textcolor{darkgray}{79.36} & \textcolor{darkgray}{73.16} & 76.13  
& \textcolor{darkgray}{97.74} & \textcolor{darkgray}{75.32} & 85.08  
& \textcolor{darkgray}{90.35} & \textcolor{darkgray}{91.49} & 90.92  
& \textcolor{darkgray}{41.72} & \textcolor{darkgray}{38.09} & \textbf{39.82} 

 \\\hline
\end{tabular}}
 \resizebox{0.95\textwidth}{!}{\begin{tabular}{c|ccc|ccc|ccc|ccc}
\hline
\multirow{2}{*}{Methods} &
  \multicolumn{3}{c|}{{SUN397}} &
  \multicolumn{3}{c|}{{DTD}} &
  \multicolumn{3}{c|}{{EuroSAT}} &
  \multicolumn{3}{c}{{UCF101}} \\
 &
  Base & New & HM & Base & New & HM & Base & New & HM & Base & New & HM \\ \hline
CLIP~\cite{radford2021learning} & \textcolor{darkgray}{69.36} & \textcolor{darkgray}{75.35} & 72.23 & \textcolor{darkgray}{53.24} & \textcolor{darkgray}{59.90} & 56.37 & \textcolor{darkgray}{56.48} & \textcolor{darkgray}{64.05} & 60.03 & \textcolor{darkgray}{70.53} & \textcolor{darkgray}{77.50} & 73.85 \\    
CoOp~\cite{zhou2022learning} & \textcolor{darkgray}{81.16} & \textcolor{darkgray}{75.08} & 78.00 & \textcolor{darkgray}{80.32} & \textcolor{darkgray}{56.52} & 66.35 & \textcolor{darkgray}{79.43} & \textcolor{darkgray}{74.26} & 76.76 & \textcolor{darkgray}{84.13} & \textcolor{darkgray}{72.96} & 78.15 \\  
CoCoOp~\cite{zhou2022conditional} & \textcolor{darkgray}{79.74} & \textcolor{darkgray}{76.86} & 78.27 & \textcolor{darkgray}{77.01} & \textcolor{darkgray}{56.00} & 64.85 & \textcolor{darkgray}{87.49} & \textcolor{darkgray}{60.04} & 71.21 & \textcolor{darkgray}{82.33} & \textcolor{darkgray}{73.45} & 77.64 \\  
ProGrad~\cite{zhu2023prompt} & \textcolor{darkgray}{81.26} & \textcolor{darkgray}{74.17} & 77.55 & \textcolor{darkgray}{77.35} & \textcolor{darkgray}{52.35} & 62.45 & \textcolor{darkgray}{90.11} & \textcolor{darkgray}{60.89} & 72.67 & \textcolor{darkgray}{84.33} & \textcolor{darkgray}{74.94} & 79.35 \\  
CLIP-Adapter~\cite{gao2024clip} & \textcolor{darkgray}{81.16} & \textcolor{darkgray}{75.08} & 78.00 & \textcolor{darkgray}{80.32} & \textcolor{darkgray}{56.52} & 66.35 & \textcolor{darkgray}{79.43} & \textcolor{darkgray}{74.26} & 76.76 & \textcolor{darkgray}{84.13} & \textcolor{darkgray}{72.96} & 78.15 \\
KgCoOp~\cite{yao2023visual} & \textcolor{darkgray}{80.29} & \textcolor{darkgray}{76.53} & 78.36 & \textcolor{darkgray}{77.55} & \textcolor{darkgray}{54.99} & 64.35 & \textcolor{darkgray}{85.64} & \textcolor{darkgray}{64.34} & 73.48 & \textcolor{darkgray}{82.89} & \textcolor{darkgray}{76.67} & 79.65 \\  
MaPLe~\cite{khattak2023maple} & \textcolor{darkgray}{80.82} & \textcolor{darkgray}{78.70} & 79.75 & \textcolor{darkgray}{80.36} & \textcolor{darkgray}{59.18} & 68.16 & \textcolor{darkgray}{94.07} & \textcolor{darkgray}{73.23} & 82.35 & \textcolor{darkgray}{83.00} & \textcolor{darkgray}{78.66} & 80.77 \\  
DePT~\cite{zhang2024dept} + MaPLe & \textcolor{darkgray}{82.90} & \textcolor{darkgray}{76.40} & 79.52 & \textcolor{darkgray}{83.87} & \textcolor{darkgray}{59.93} & 69.91 & \textcolor{darkgray}{94.43} & \textcolor{darkgray}{76.23} & 84.36 & \textcolor{darkgray}{86.87} & \textcolor{darkgray}{78.10} & \textbf{82.25} \\ 
MMA~\cite{yang2024mma} & \textcolor{darkgray}{82.27} & \textcolor{darkgray}{78.57} & 80.38  
& \textcolor{darkgray}{83.20} & \textcolor{darkgray}{65.63} & 73.38  
& \textcolor{darkgray}{85.46} & \textcolor{darkgray}{82.34} & 83.87  
& \textcolor{darkgray}{86.23} & \textcolor{darkgray}{80.03} & 82.20 \\  
\rowcolor{tabhighlight}  
RAda & \textcolor{darkgray}{80.38} & \textcolor{darkgray}{75.97} & 78.11 & \textcolor{darkgray}{79.17} & \textcolor{darkgray}{58.70} & 67.42 & \textcolor{darkgray}{90.40} & \textcolor{darkgray}{74.72} & 81.82 & \textcolor{darkgray}{84.80} & \textcolor{darkgray}{76.74} & 80.57 \\   
\rowcolor{tabhighlight}  
RAda + Prompt & \textcolor{darkgray}{82.38} & \textcolor{darkgray}{77.30} & 79.76 & \textcolor{darkgray}{83.28} & \textcolor{darkgray}{60.87} & 70.33 & \textcolor{darkgray}{94.27} & \textcolor{darkgray}{84.69} & \textbf{89.22} & \textcolor{darkgray}{86.21} & \textcolor{darkgray}{78.15} & 81.98\\
\rowcolor{tabhighlight}  
RAda + Adapter & \textcolor{darkgray}{82.58} & \textcolor{darkgray}{78.77} & \textbf{80.63} 
& \textcolor{darkgray}{82.06} & \textcolor{darkgray}{67.15} & \textbf{73.86}  
& \textcolor{darkgray}{96.48} & \textcolor{darkgray}{74.69} & 84.20  
& \textcolor{darkgray}{85.78} & \textcolor{darkgray}{78.26} & 81.85

\\ \hline
\end{tabular}}
\vspace{-0.5 cm}
\end{table*}

\section{Experiments}
We conduct experiments in three fine-tuning strategies to evaluate RAda on a same NVIDIA A100 (40GB RAM). The CLIP ViT-B/16 from OpenAI is used as backbone. 

\subsection{RAda in FFT}
We extend the training paradigm in LP-PT~\cite{kumar2022fine} for RAda in this setting, where we first train $\theta_m$ (\ie referred to as RAda) and then jointly updates all parameters (\ie referred to as RAda-FT). Note in FFT, a linear classifier is used to replace the text encoder, and its weight, initialized by the text features, can then be regarded as the evolving text information.
We compare with several different FFT ideas, namely LP that only updates the new classifier, FT that updates both the image encoder and the new classifier, LP-FT~\cite{kumar2022fine}, L2-SP~\cite{xuhong2018explicit} that ensures similarity between pretrained and finetuned models, and FLYP~\cite{goyal2023finetune} that mimics the pretraining pipeline of CLIP in FFT. Cross-entropy loss is utilized, except for FLYP where the original contrastive loss is adopted.  

\noindent\textbf{Datasets and implementation details.} We use 6 datasets for evaluations: ImageNet~\cite{deng2009imagenet} is considered the ID dataset for fine-tuning the model, and 5 standard out-of-distribution (OOD) datasets (\ie, ImageNetV2~\cite{recht2019ImageNet}, ImageNet-R~\cite{hendrycks2021many}, ImageNet-A~\cite{hendrycks2021natural}, ImageNet-Sketch~\cite{wang2019learning}, and ObjectNet~\cite{barbu2019objectnet}) are used for evaluations.
%
Following~\cite{wortsman2022robust,goyal2023finetune}, we use a batch size of 512 and train for 10 epochs. For RAda, we report results after 10 epochs; for RAda-FT, we initialize the model with the pretrained weights of RAda after 5 epochs and then perform RAda-FT for another 5 epochs. Similar to~\cite{kumar2022fine}, we use diverse learning rates for RAda and RAda-FT (\ie 0.004 and 0.000004). Other settings, such as optimizers, weight decay, \etc are inherited from~\cite{wortsman2022robust}, same as~\cite{goyal2023finetune}. Please see our supplementary material for details. 

\noindent\textbf{Experimental results.} We list the results in Table~\ref{tab fft}.
%
Similar to LP, RAda underperforms FT on ID data but shows better results across most OOD datasets, except for ImageNetV2 which is close to ID. This aligns with prior theory \cite{kumar2022fine} that training later layers enhances generalization by preserving pretrained features, especially when ID and OOD distributions are substantially distinct. Since RAda operates in even later layers than the classifier, the feature-preserving theory can be further validated when comparing RAda and LP. 
Moreover, we notice that RAda-FT improves performance in all OOD datasets for FT and outperforms prior arts~\cite{goyal2023finetune,kumar2022fine}. This is because RAda-FT can better balance the tradeoff between overfitting and retaining both pretrained text and visual features, a benefit cannot be achieved by focusing different modalities in isolation.
These findings underscore the effectiveness of prioritizing final decision-making process in VLMs as a targeted FFT strategy.

\subsection{RAda in EFT}
\label{sec eft}
This experiment aims to evaluate whether the adapted rational matrix can contribute when the encoders are fixed during fine-tuning. 
Besides the baseline CLIP model, we compare our idea with some recent arts that specifically designed for EFT: CLIP-Adapter~\cite{gao2024clip}, CoOp~\cite{zhou2022learning}, CoCoOp~\cite{zhou2022conditional}, ProGrad~\cite{zhu2023prompt}, KgCoOp~\cite{yao2023visual}, MaPLe~\cite{khattak2023maple}, DePT~\cite{zhang2024dept} on the basis of MaPLe, and MMA~\cite{yang2024mma}. Results are directly cited from the paper except for~\cite{gao2024clip}, which is reimplemented in our device using the provided code.

\noindent\textbf{Datasets and implementation details.}
We test the methods with the base-to-novel generalization setting. A total of 11 datasets are utilized, including, 2 used for classification on generic objects, \ie ImageNet and Caltech101~\cite{fei2004learning}; 5 used for fine-grained classification, \ie OxfordPets~\cite{parkhi2012cats}, StanfordCars~\cite{krause20133d}, Flowers102~\cite{nilsback2008automated}, Food101~\cite{bossard2014food}, and FGVCAircraft~\cite{maji2013fine}; an scene recognition dataset SUN397 \cite{xiao2010sun}; a action recognition dataset UCF101~\cite{soomro2012ucf101}; a texture classification dataset DTD~\cite{cimpoi2014describing}; and a satellite imagery recognition EuroSAT~\cite{helber2019eurosat}. 
Batch size, learning rate, and epoch are fixed as 1, 0.0009, and 13 for all datasets, and we use 16 shots per class for the source data, same as~\cite{zhou2022conditional}. 

\begin{table*}[ht]
\centering
    \caption{\textbf{Comparisons with representative TTT methods regarding DG performance in four distribution shifts}. Here ``pretrained" denotes whether the model is pretrained in ImageNet. RAda shows comparable effectiveness against arts specialized in the TTT task.} 
    \vspace{-0.3 cm}
 \scalebox{0.95}{
    \begin{tabular}{lcccccc}
    \toprule
    &pretrained & ImageNet V2 & ImageNet Sketch & ImageNet A &  ImageNet R  & OOD Avg.\\
    \midrule
    CLIP & \xmark & 60.86 & 46.09 & 47.87 & 73.98 & 57.20 \\
    \hline
    TPT~\cite{shu2022test} & \xmark & 64.35 & 47.94 & 54.77 & 77.06 & 60.81 \\
    CoOp~\cite{zhou2022learning}+TPT & \cmark & \textbf{66.83} & 49.29  & 57.95  & 77.27  & 62.84 \\
    CoCoOp~\cite{zhou2022conditional}+TPT & \cmark & 64.85 & 48.27 & 58.47 & 78.65 & 62.61  \\
    MaPLe~\cite{khattak2023maple}+TPT & \cmark & 64.87 & 48.16 & 58.08  & 78.12 & 62.31\\
    PromptAlign~\cite{abdul2024align} & \cmark & 65.29 & \textbf{50.23} & {59.37}  &79.33 & {63.55} \\
    \rowcolor{tabhighlight}
    RAda & \xmark &64.10 &49.36 &61.17 &79.35 &63.50 \\
    \rowcolor{tabhighlight}
    RAda$^{\dag}$ & \cmark &65.10 &49.45 &\textbf{62.72} &\textbf{79.75} &\textbf{64.26}\\
    \bottomrule
    \end{tabular}}
    \label{tab ttt}
    \vspace{-0.5cm}
\end{table*}

\noindent\textbf{Experimental results.}
Results in Table~\ref{tab b2n} show that RAda increases the average base accuracy by nearly 13pp for the baseline CLIP without compromising its performance in unseen classes-a tradeoff observed in many other methods~\cite{gao2024clip,zhou2022learning,zhou2022conditional,zhu2023prompt,yao2023visual} where base accuracies are improved at the cost of novel class performances. 
These results validate the effectiveness of RAda in fast adaptation to new data while preserving the generalizability of the original CLIP.

While RAda alone does not achieve state-of-the-art performance, it remains highly competitive: among the compared arts, RAda is outperformed only by arts that specifically modify intermediate representations~\cite{khattak2023maple, zhang2024dept,yang2024mma}. 
When also interfering the intermediate representations within RAda, we show its performance can be further boosted. For instance, adopting prompt tuning for RAda (RAda + Prompt) achieves better performance than that utilized in MaPLe, while combining adapter tuning for RAda (RAda + Adapter) yields best average accuracy.
This synergy with complementary strategies highlights RAda is orthogonal to existing EFT paradigms (\eg prompt or adapter tuning).
Collectively, these results affirm RAda's potential as a competitive approach in the EFT setting.

%
%

\subsection{RAda in TTT}
Different from other strategies, TTT can only access the unlabeled test data in updating. We compare RAda with two recent methods, \ie TPT~\cite{shu2022test} and PromptAlign~\cite{abdul2024align}, within this setting. Both these two methods are developed based on the prompt tuning paradigm. Specifically, TPT extends CoOp by updating the text prompts with the entropy minimization objective, and PromptAlign extends the idea in~\cite{khattak2023maple} by including an additional distribution alignment regularization to refine both text and visual prompts.  

\noindent\textbf{Datasets and implementation details.}
Same as previous works~\cite{shu2022test,abdul2024align}, we use the 4 OOD datasets (\ie ImageNetV2, ImageNet-R, ImageNet-A, and ImageNet-Sketch) for evaluation. For every test sample, we obtain 63 of its augmented view using the same augmentation strategies in~\cite{abdul2024align} to form a batch of 64 samples, among which, we select the top 10\% confident predictions with the lowest entropy and compute the entropy loss in Eq.~\eqref{eq ttt} for the sub-batch. The offline TTT updating strategy~\cite{sun2020test} is adopted where the weights are initialized to the original state for each sample, so that the order of the arrived data does not affect the result. The learning rate is fixed as 0.0008 for all datasets, and we perform three updating steps for each of the test sample.

\noindent\textbf{Experimental results.}
As shown in Table~\ref{tab ttt}~\footnote{Results for some datasets are different for CLIP in Table~\ref{tab fft} and~\ref{tab ttt} is due to the different prefixed text prompts $\mathbf{p}$ used in these two settings.}, RAda enhances the baseline CLIP across all evaluated datasets, outperforming the naive TPT in 3 datasets and leading the average accuracy by 2.7pp. These observations validate the effectiveness of focusing the final decision-making process for adaptation in test-time.
When compared to the recent PromptAlign, RAda demonstrates strengths on half of the evaluated datasets with comparable average results, which is achieved without leveraging pretrained information in ImageNet. 
When using the same pretrained information, we observe that the performance of RAda can be further enhanced, leading all compared arts in average performance. Combined with the demonstrated efficiency of RAda in Table~\ref{tab efficient}, these results validate RAda as a strong competitor for the TTT application, even against task-specific methods.
%

\section{Analysis}

\subsection{Ablation Studies}
\label{sec ablation}
We evaluate the effectiveness of our designs using the base-to-new generalization setting in EFT, where the settings are the same as that detailed in Sec.~\ref{sec eft}. Please refer to our supplementary material for more ablation studies.

\noindent\textbf{Effectiveness of the regularization term.}
To assess the impact of the regularization term on the model performance, we evaluate RAda with and without adding $\mathcal{L}_{reg}$ in the overall objective. As shown in the 2nd row in Table \ref{tab analysis}, the base accuracy for the baseline CLIP can be improved across both settings, indicating that the the regularization term does not degrade performance on familiar classes. Meanwhile, we observe the inclusion of $\mathcal{L}_{reg}$ significantly improves the generalization performance, \ie the accuracy on novel classes, achieving an improvement of $2.5$pp. These results highlight the importance of the regularization term in maintaining the training effectiveness of RAda without compromising the superior zero-shot capability of CLIP.

\noindent\textbf{Different settings for implementing $\mathcal{F}_{m}$.} We use a multi query attention layer for implementing $\mathcal{F}_{m}$. In this section, we assess the impact of adopting different settings for $\mathcal{F}_{m}$. Namely, we first try using an MLP layer to replace the attention-based rational adapter (\ie MLP for $\mathcal{F}_{m}$), and then use the following query settings in the final attention layer: rational matrix $\mathbf{R}$, and its combination with either the image features $\overline{\mathbf{f}}$ or text features $\overline{\mathbf{h}}$. 

As listed in 3rd-6th rows from Table~\ref{tab analysis}, although using an MLP for $\mathcal{F}_{m}$ can improve the baseline, it performs inferior to using the attention-based design. This is mainly because the attention layer encourages the interactions of different to-be-distinguished classes, while the MLP layer can only process different classes independently. These results justify our motivation of using an attention layer to implement $\mathcal{F}_{m}$.
Meanwhile, we note that incorporating additional information from either text or visual modality improves the performance than rely solely on the rational matrix to act as query, key, and value. This is because fusing $\overline{\mathbf{f}}$ and $\overline{\mathbf{h}}$ may obscure their specific patterns, resulting $\mathbf{R}$ to not include all information.
%
%
%
In comparison, our multi query setting offers best results in both base and novel class accuracies, underscoring the importance of leveraging all available information in the rational matrix calibration process.

\noindent\textbf{Effectiveness of the fused information.} 
To verify if leveraging fused information is superior to ideas that consider the different modalities in isolation, we compare RAda with variants that attach attention layers to encoders, which modify the original image or text features through a learned mask from the the attached attention layer. Specifically, three variants are compared: ``$\mathcal{F}_t$ + attn" and ``$\mathcal{F}_v$ + attn", where an attention layer is attached at the end of the text and visual encoders, respectively; and a combined variant that applies separate attention layers to each encoder.

As seen in 7th-9th rows in Table~\ref{tab analysis}, adding an attention layer enhances adaptation to the training distribution, whether attached to the vision or text encoder, as the baseline accuracies increase across all three variants. Notably, the variant with attention layers attached to both the text and visual encoders performs the best, likely due to the incorporation of adapted information from both modalities.
However, despite using the same mask regularization, attaching attention layers solely to the encoders does not preserve the superior zero-shot ability of the original CLIP, as all these variants show marked declines in novel class accuracy compared to RAda.
These results verifies the superiority of using the final fused information in fine-tuning compared with these variants that process different modalities in isolation. In the supplementary material, we further provide theoretical explanations for this observation.

\begin{table}[t]
    \centering
    \caption{Comparisons of RAda with its different variants.}
    \vspace{-0.3 cm}
    \scalebox{0.9}{
    \begin{tabular}{cccc}
    \toprule
         Variants &Base acc. &New acc. & HM\\
         \midrule
         Baseline &69.34 &74.22 &71.70\\
         \hline
         W/O $\mathcal{L}_{reg}$ & 81.38 & 71.58 & 76.16 \\
         \hline
         MLP for $\mathcal{F}_{m}$ & 77.51 & 73.18 & 75.28 \\
         query=$\{\mathbf{R}\}$ & 81.99 & 73.89 & 77.72\\
         query=$\{\overline{\mathbf{h}}, \mathbf{R}\}$ & 81.93 & 74.07 & 77.80 \\
         query=$\{\overline{\mathbf{f}}, \mathbf{R}\}$ & 82.03 & 73.93 & 77.76 \\
         \hline
         (a) $\mathcal{F}_t$ + attn  &82.56 &69.62 &75.53 \\
         (b) $\mathcal{F}_v$ + attn &81.96 &71.18 &76.19  \\
         (a) + (b) &82.09 &72.46 &76.97 \\
         \hline
         \rowcolor{tabhighlight}
         RAda (query=$\{\overline{\mathbf{h}}, \overline{\mathbf{f}}, \mathbf{R}\}$) &82.16 &74.14 &77.94\\
         \bottomrule
    \end{tabular}}
    \label{tab analysis}
    \vspace{-0.5 cm}
\end{table}

\subsection{Comparisons with MaPLe~\cite{khattak2023maple}}
\label{sec maple}
The closest conceptual counterpart to RAda is MaPLe \cite{khattak2023maple}, which similarly seeks to leverage fused information during fine-tuning. Key differences between these two methods are as follows. 
(1) \textbf{Efficiency.} MaPLe updates intermediate prompts across layers, incurring high memory and computational costs. In contrast, RAda’s adaptation is confined to the final output stage, requiring minimal memory and computational resources. Efficiency comparisons in Table~\ref{tab efficient} solidifying its lightweight advantage over MaPLe.
(2) \textbf{Effectiveness.} The two methods exhibit divergent performance across the three mainstream fine-tuning settings. Specifically, while MaPLe achieves a 0.6pp average advantage over RAda in EFT (78.55 \vs 77.94), it is not applicable in the FFT setting, where RAda achieves leading performance, and it underperforms RAda by nearly 2pp in the TTT setting (62.31 \vs 64.26). These results highlight that the strengths of each method vary depending on the fine-tuning scenario and that RAda exhibits greater versatility, delivering consistently competitive results in most settings.
(3) \textbf{Applicability.} Inheriting limitations from vision prompt tuning \cite{jia2022visual}, MaPLe is restricted to the transformer-based image encoders. In contrast, RAda thrives with arbitrary encoder structures (see results in the supplementary material), further demonstrating its encoder-agnostic advantage.

\begin{table}[t]
\caption{Efficiency comparisons with MaPLe~\cite{khattak2023maple} for a single updating step in fine-tuning CLIP with a batch size of 1.} 
\vspace{-0.3 cm}
    \centering
    \scalebox{0.89}{
    \begin{tabular}{ccccc}
    \toprule
         Setting& Method&FPS($\uparrow$) &Memory($\downarrow$) &GFLOPs($\downarrow$)\\
         \hline 
         \multirow{2}*{EFT}&MaPLe &  27.20 & 1.09GB & 206.19\\
         &RAda & 50.11 & 0.49GB & 17.67\\
         \hline
         \multirow{2}*{TTT}&MaPLe & 34.51 & 1.31GB & 191.75 \\
         &RAda&68.39 &0.61GB &17.26\\
         \bottomrule
    \end{tabular}}
    \label{tab efficient}
    \vspace{-0.5 cm}
\end{table}
\section{Discussion and Conclusion}
\noindent\textbf{Future works.}
While RAda has demonstrated its effectiveness in different fine-tuning settings,a promising extension is to apply it in the pretraining phase,where a richer fused representation from massive data can be utilized. 
Meanwhile, besides the classification task, the literature could also consider exploring RAda on other downstream applications within the VLM contexts, such as image captioning and visual question answering, where a fine-grained understanding of relationships between different modalities is essential for improving the performance.

\noindent\textbf{Conclusion.}
This paper proposes a new rational adaptation method to effectively focus on the final decision-making process of CLIP, aiming to explicitly leverage fused representations from different modalities for improved performance.
The idea is achieved by a simple implementation that attaches an additional attention layer at the end to learn a mask that can adaptively decide contributions for different rational elements.
Through comprehensive experiments across various settings, we find the proposed idea can serve as a versatile fine-tuning strategy, consistently benefiting the baseline and competing favorably against existing arts.

\section*{Acknowledgment.} 
This research is supported by the MBZUAI-WIS Joint Program for AI Research.

{
    \small
    \bibliographystyle{ieeenat_fullname}
    \bibliography{rada}
}

\clearpage
\appendix
\setcounter{page}{1}
\maketitlesupplementary
\setcounter{section}{0}
In this supplementary material, we provide,

\begin{enumerate}[labelsep=0.2em]
    \item Theoretical support for using fused information in Sec.~\ref{sec theory}.

    \item Visualizations for RAda in Sec.~\ref{sec visual}.

    \item Detailed settings of RAda in the FFT setting in Sec.~\ref{sec fftsetting}.

    \item Extending other fine-tuning ideas in FFT in Sec.~\ref{sec morefftexp}.

    \item Ablation studies on $\mathcal{L}_{reg}$ in Sec.~\ref{sec weight}.

    \item Experiments with different backbones in Sec.~\ref{sec backbones}.

    \item Experiments with different VLMs in Sec.~\ref{sec vlms}.

    \item Experiments for using more attention layers in Sec.~\ref{sec moreattn}.
\end{enumerate}

\section{Theoretical Support for Utilizing Fused Information Over Isolated Representation}
\label{sec theory}
Our empirical observations indicate that adapting the rational matrix yields better performance than adapting the different modalities in isolation.
In this section, we provide theoretical explanations for justification the selection of utilizing the final fused information rather than the isolated representations in fine-tuning. In particular, we demonstrate the advantage of RAda against three fine-tuning ideas (\ie separately adapt the image features or text features in isolation or jointly adapt both features). The explanation framework is grounded in information theory and statistical sufficiency. Below is a step-by-step proofs.

Given the random sampled class label $\mathbf{Y}$, image embedding $\mathbf{f} \in \mathbb{R}^{D}$, text embedding $\mathbf{h} \in \mathbb{R}^{K \times D}$, rational matrix $\mathbf{R} \in \mathbb{R}^{K \times D}$ (each element defined as $\mathbf{R}_{i,j} = \textbf{f}_j \cdot \textbf{h}_{i,j}$ \footnote{The two embeddings are both normalized in this section.}), we present the adaptation of image or text embedding as learning task-specific transformations: $\mathbf{f} \mapsto \mathbf{f} \circ \mathbf{M}_\mathbf{f}$, $\mathbf{h} \mapsto \mathbf{h} \circ \mathbf{M}_\mathbf{h}$, and similar for the rational matrix: $\mathbf{R} \mapsto \mathbf{R} \circ \mathbf{M}$, with $\circ$ element-wise product and $\mathbf{M}_\mathbf{f} \in \mathbb{R}^{D}, \mathbf{M}_\mathbf{h} \in \mathbb{R}^{K \times D}$, and $\mathbf{M} \in \mathbb{R}^{K \times D}$ being learnable parameters. 
We first have,
\begin{lemma}
\label{lemma1}
    The rational matrix $\mathbf{R}$ is a sufficient statistic of $\mathbf{Y}$. Formally, by the definition,
    \begin{equation}
        p(\mathbf{Y}\vert \mathbf{f}, \mathbf{h}) = p(\mathbf{Y}\vert \mathbf{R}),
    \end{equation}
\end{lemma}
where $p(\mathbf{Y}\vert \mathbf{f}, \mathbf{h})$ denotes that the prediction in CLIP relies on both the image and the text embeddings.
\begin{proof}
    In the CLIP model, the prediction rule depends only on the inner products $\sum_j \mathbf{R}_{i,j}$, which are functions of $\mathbf{R}$. Thus, the likelihood $p(\mathbf{Y} \vert \mathbf{f}, \mathbf{h})$ depends on $\mathbf{f}$ and $\mathbf{h}$ only through $\mathbf{R}$. Therefore, $\mathbf{R}$ is a sufficient statistic for $\mathbf{Y}$.
\end{proof}
Given $\mathbf{R}$ is sufficient for $\mathbf{Y}$, we thus have equality between mutual informations: $I(\mathbf{Y};\mathbf{R}) = I(\mathbf{Y;\mathbf{f}, \mathbf{h}})$.

\begin{lemma}
\label{lemma2}
    Adapting $\mathbf{R}$ achieves mutual information with $\mathbf{Y}$ no less than adapting $\mathbf{f}$ or $\mathbf{h}$. In particular,
    \begin{equation}
    \small
    \begin{aligned}
    \label{eq lemma2}
        I(\mathbf{Y;\mathbf{R} \circ \mathbf{M}}) \geq  \max \{I(\mathbf{Y};\mathbf{f} \circ \mathbf{M}_\mathbf{f}, \mathbf{h}),  I(\mathbf{Y};\mathbf{h} \circ \mathbf{M}_\mathbf{h}, \mathbf{f})\}.
    \end{aligned}
    \end{equation}
\end{lemma}
\begin{proof}
    Revisiting the first term in RHS of Eq.~\eqref{eq lemma2}, we can represent it as: $ I(\mathbf{Y};\mathbf{f} \circ \mathbf{M}_\mathbf{f}, \mathbf{h}) = I(\mathbf{Y}; \mathbf{R} \circ (\mathbf{M}_{\mathbf{f}} \otimes \mathbf{1}_K^\intercal))$, with $\otimes$ denotes the Kronecker product, and $\mathbf{M}_{\mathbf{f}} \otimes \mathbf{1}_K^\intercal$ refers replicate $\mathbf{M}_{\mathbf{f}}$ across rows.
    By the data processing inequality (DPI)~\cite{cover1999elements}, we have,
    \begin{equation}
    \label{eq dpilemma2}
    I(\mathbf{Y}; \mathbf{R} \circ \mathbf{M}) \geq I\left(\mathbf{Y}; \mathbf{R} \circ (\mathbf{M}_{\mathbf{f}} \otimes \mathbf{1}_K^\intercal)\right),
    \end{equation}
    since $\mathbf{R} \circ (\mathbf{M}_{\mathbf{f}} \otimes \mathbf{1}_K^\intercal)$ can be regarded as a deterministic function of $\mathbf{R} \circ \mathbf{M}$ (by constraining $\mathbf{M}$ to be column-wise).
    If we constrain $\mathbf{M} = \mathbf{M}_{\mathbf{f}} \otimes \mathbf{1}_K^\intercal$, where the task-relevant information in $\mathbf{R}$ is uniformly distributed across rows within each column, then:
    $
    I(\mathbf{Y}; \mathbf{R} \circ \mathbf{M}) = I\left(\mathbf{Y}; \mathbf{R} \circ (\mathbf{M}_{\mathbf{f}} \otimes \mathbf{1}_K^\intercal)\right).
    $
    
    The same goes for adapting $\mathbf{h}$ with $\mathbf{M}_{\mathbf{h}}$, with $ I(\mathbf{Y};\mathbf{h} \circ \mathbf{M}_\mathbf{h}, \mathbf{f}) = I(\mathbf{Y}; \mathbf{R} \circ \mathbf{M}_{\mathbf{h}})$, due to the uniform natural of text embeddings across all samples, $\mathbf{R} \circ \mathbf{M}_{\mathbf{h}}$ can be regarded as a deterministic function of $\mathbf{R} \circ \mathbf{M}$ (by constraining $\mathbf{M}$ to be sample-wise). We thus have,
    \begin{equation}
    \begin{aligned}
    \label{eq adapth}
        I(\mathbf{Y;\mathbf{R} \circ \mathbf{M}}) \geq  I(\mathbf{Y};\mathbf{h} \circ \mathbf{M}_\mathbf{h}, \mathbf{f}).
    \end{aligned}
    \end{equation}
    
    In all, equality in Eq.~\eqref{eq lemma2} holds for these constrained cases. Otherwise, we will have LHS larger than RHS in Eq.~\eqref{eq lemma2}. 
\end{proof}

\begin{lemma}
\label{lemma3}
    Adapting $\mathbf{R}$ achieves higher mutual information with $\mathbf{Y}$ than adapting $\mathbf{h}$ and $\mathbf{f}$ jointly. In particular,
    \begin{equation}
    \begin{aligned}
    \label{eq lemma3}
        I(\mathbf{Y;\mathbf{R} \circ \mathbf{M}}) \geq I(\mathbf{Y};\mathbf{h} \circ \mathbf{M}_\mathbf{h}, \mathbf{f} \circ \mathbf{M}_\mathbf{f}).
    \end{aligned}
    \end{equation}
\end{lemma}
\begin{proof}
    By the DPI, we have,
    \begin{equation}
        I(\mathbf{Y};\mathbf{h}, \mathbf{f} \circ \mathbf{M}_\mathbf{f}) \geq I(\mathbf{Y};\mathbf{h} \circ \mathbf{M}_\mathbf{h}, \mathbf{f} \circ \mathbf{M}_\mathbf{f}),
    \end{equation}
    where the equality holds when $\mathbf{M}_{\mathbf{h}}$ is invertible.
    Combining with Lemma~\ref{lemma2}, we thus can complete the proof. 
\end{proof}

Lemma~\ref{lemma2} and~\ref{lemma3} demonstrate that adapting the rational matrix results in mutual information no less than adapting $\mathbf{f}$, $\mathbf{h}$, or both $\mathbf{f}$ and $\mathbf{h}$. 
According to the information bottleneck principle~\cite{tishby2000information}, a higher mutual information between the label and an intermediate representation generally correlates with better predictive performance. Given the compression performance (\ie generalizibility) of the model can be largely preserved via an all-one regularization for $\mathbf{M}$ in our implementation, it is not surprise that leveraging the fused representation can be more beneficial than utilizing the individual modalities in isolation. These observations align with our empirical observations in the ablation studies, and they further justify our motivation for adapting the rational matrix to achieve effective fine-tuning.

\begin{figure}[t]
    \centering
    \includegraphics[width=0.97\linewidth]{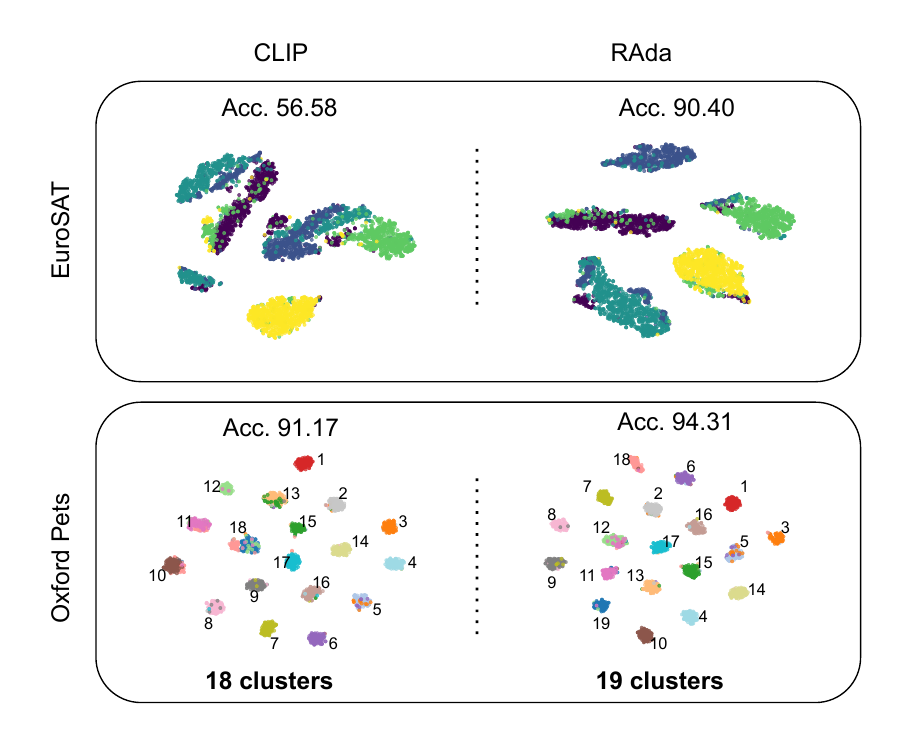}
    \vspace{-0.3cm}
    \caption{T-SNE~\cite{van2008visualizing} plots of the Rational Matrix from CLIP and RAda in the EuroSat (5 classes) and OXfordPets (19 classes) datasets. The adapted rational matrix in RAda shows clearer and more precise separation than that in the original CLIP.}
    \label{fig:t-SNE}
    \vspace{-0.2cm}
\end{figure}

\section{Visualization}
\label{sec visual}
We present 2D t-SNE plots \cite{van2008visualizing} of $\mathbf{M} \circ \mathbf{R}$ (corresponding to RAda) and $\mathbf{R}$ (representing the original CLIP) to illustrate the behavior of the learned mask $\mathbf{M}$. As shown in Figure \ref{fig:t-SNE}, the adapted $\mathbf{R}$ in the EuroSAT dataset exhibits enhanced differentiability, with tighter clusters indicating improved class separability compared to that of CLIP.
Additionally, the plot of $\mathbf{M} \circ \mathbf{R}$ for OxfordPets reveals $19$ clusters, matching the total class count in the dataset, while the original $\mathbf{R}$ from CLIP shows only $18$ classes. This distinction aligns with the enhanced classification performance achieved through $\mathbf{M}$. These findings validate the effectiveness of adapting the decision-making process in achieving improved predictions within a VLM.

We also present the distribution of the values in the mask $\mathbf{M}$ for the two datasets. The plots in Figure~\ref{fig:plots} show that, for both datasets, the mask values exhibit a mean value of approximately \(1.0\), with the majority of the weights centered around \(1\) to form a normal distribution, and the mask value can reach as high as \(3\), indicating the varying contributions of different rational elements after adaptation. We hope this finding can inspire future research to develop more effective learning objective for adapting the rational matrix. 
We present an example of the adaptation process to better illustrate how the original decision matrix got shifted by the process. We show heatmaps of the mask $\textbf{M}$, the rational matries $\textbf{R}$, and $\textbf{M}\circ \textbf{R}$ in Figure~\ref{fig:heatmap} (left to right), where $\textbf{M}\circ \textbf{R}$ shows more evident classification clue than $\textbf{R}$ with larger values in the 2nd last column, suggesting the rational adaptation process helps the model to capture more confident outputs.

\begin{figure}[t]
    \centering
    \includegraphics[width=0.97\linewidth]{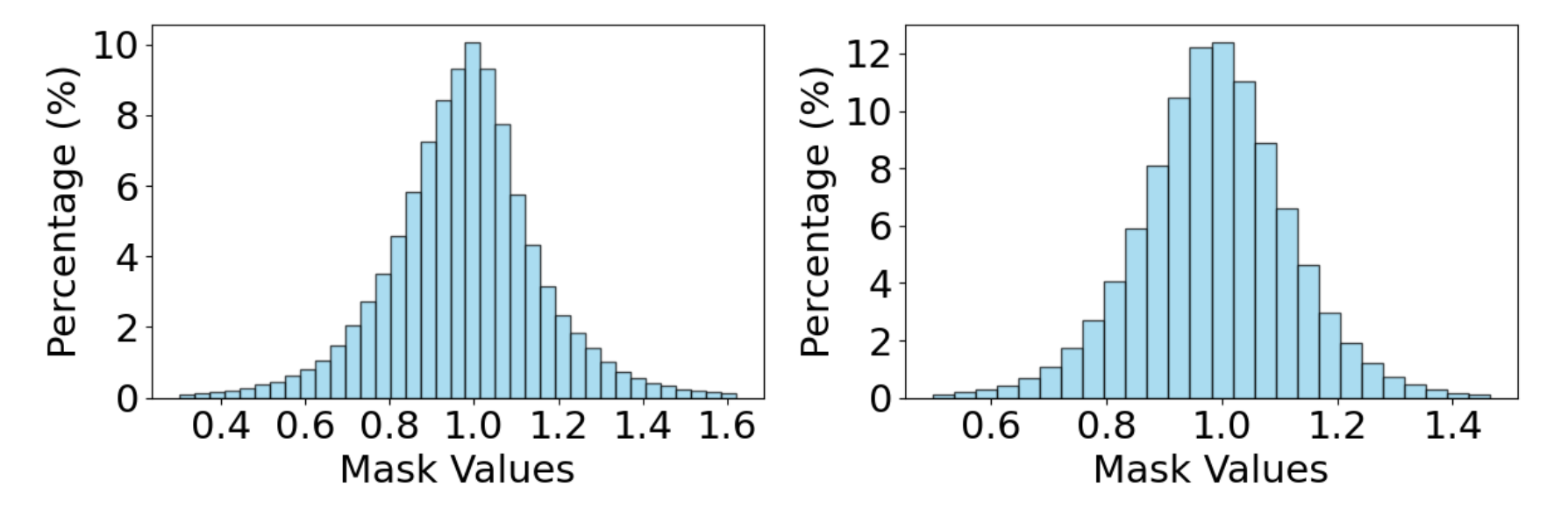}
    \vspace{-0.4cm}
    \caption{Distributions of values in the learned mask $\mathbf{M}$ for EuroSAT (\ie left figure) and Oxford Pets~(\ie right figure) datasets.}
    \label{fig:plots}
    \vspace{-0.2cm}
\end{figure}

\begin{figure}[t]
    \centering
    \includegraphics[width=0.95\linewidth]{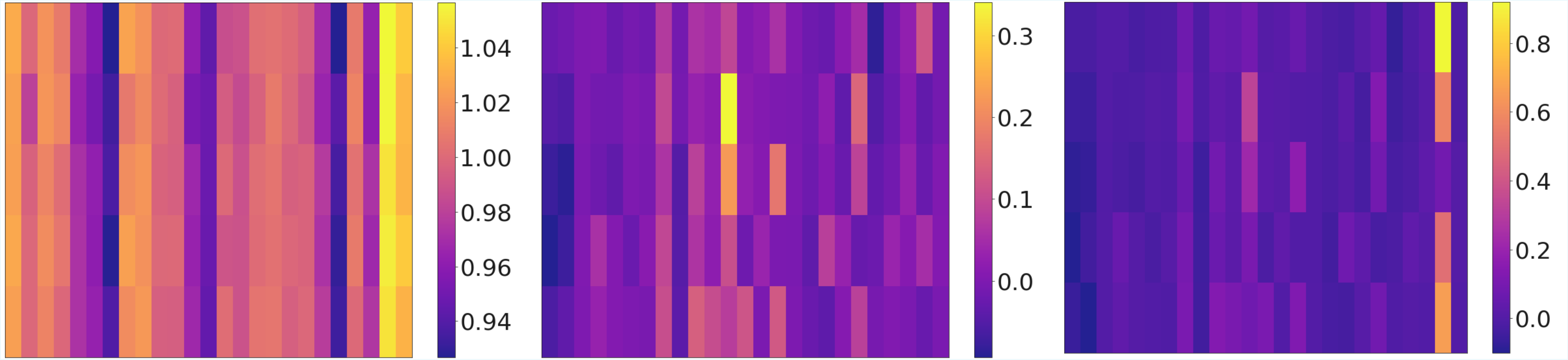}\\
    \centering{(a) $\textbf{M}$~~~~~~~~~~~~~~~~~~~~~~(b) $\textbf{R}$~~~~~~~~~~~~~~~~(c) $\textbf{M} \circ \textbf{R}$}
    \vspace{-0.2cm}
    \caption{Heatmaps of the rational adaptation process.}
    \label{fig:heatmap}
    \vspace{-0.5cm}
\end{figure}

\section{Detailed Settings in FFT}
\label{sec fftsetting}
Our implementation in the FFT setting consists of two consecutive parts, first updating the rational adapter (\ie RAda) and then updating all learnable parameters (\ie RAda-FT). This section provides more details regarding the objectives and hyper-parameter settings for the two parts.

First, for RAda, the objective is,
\begin{equation}
\small
    \label{eq radaobj}
    \begin{aligned}
    \arg \min_{\theta_m} \Vert \mathbf{M} - \textbf{1}_{K\times D} \Vert^2 - \log \frac{\exp(<\textbf{1}_D, (\mathbf{M}\circ\mathbf{R'})_{\ast}>)}{\sum_{i=1}^K \exp(<\textbf{1}_D, (\mathbf{M}\circ\mathbf{R'})_{i}>)}, \\
    \text{s.t.}~~\mathbf{R'}^\top = 
        \begin{bmatrix}
        \overline{\mathbf{f}}_1 {\mathbf{W}}_{\{1, 1\}}  &\overline{\mathbf{f}}_1 {\mathbf{W}}_{\{1, 2\}}  &$\dots$ &\overline{\mathbf{f}}_1 {\mathbf{W}}_{\{1, K\}} \\
        \overline{\mathbf{f}}_2 {\mathbf{W}}_{\{2, 1\}}  &\overline{\mathbf{f}}_2 {\mathbf{W}}_{\{2, 2\}} &$\dots$ &\overline{\mathbf{f}}_2 {\mathbf{W}}_{\{2, K\}} \\
        \vdots &\vdots & \ddots &\vdots \\
        \overline{\mathbf{f}}_D {\mathbf{W}}_{\{D, 1\}}  &\overline{\mathbf{f}}_D {\mathbf{W}}_{\{D, 2\}}  &$\dots$ & \overline{\mathbf{f}}_D {\mathbf{W}}_{\{D, K\}} \\
        \end{bmatrix},
    \end{aligned}
\end{equation}
where the first term is the smooth regularization for the mask, and the second term is the main classification loss. Since the text encoder is replaced with a linear classifier, we use the weight of the classifier (\ie $\mathbf{W}\in\mathbb{R}^{K\times D}$, which is initialized by the text feature $\overline{\mathbf{h}}$) to compute the corresponding rational matrix $\mathbf{R'}$.
We train it for 10 epochs with the learning rate of 0.004 and batch size of 512. Default settings from~\cite{wortsman2022robust} are adopted, where the AdamW optimizer~\cite{loshchilov2017decoupled} is utilized; weight decay is set to be 0.1; the same warmup learning strategy is also utilized. 

Second, for RAda-FT, the objective is,
\begin{equation}
\small
    \label{eq radaftobj}
    \begin{aligned}
    \arg \min_{\{\theta'_m, \theta_v, \mathbf{W}\}} - \log \frac{\exp(<\textbf{1}_D, (\mathbf{M}\circ\mathbf{R'})_{\ast}>)}{\sum_{i=1}^K \exp(<\textbf{1}_D, (\mathbf{M}\circ\mathbf{R'})_{i}>)},
        \end{aligned}
\end{equation}
where $\theta'_m$ is the rational adapter trained after 5 epochs with the objective in~Eq.~\eqref{eq radaobj}. We train RAda-FT with Eq.~\eqref{eq radaftobj} for 5 epochs using the same settings as RAda, except for the learning rate, which is set as 0.000004. 

\begin{table}[t]
\caption{Extending CLIP-Adapter~\cite{gao2024clip} in the FFT setting. Results with $\dag$ are reevaluated in our device, others are from FLYP~\cite{goyal2023finetune}.}
\vspace{-0.3 cm}
\centering
\scalebox{0.78}{
\begin{tabular}{@{}cccccccc@{}}
\toprule
  & \multicolumn{7}{c}{Training in Imagenet}  \\ 
  \cmidrule(l){2-8} 
\multirow{-2}{*}{Methods} & ID   & Im-V2  & Im-R &Im-A &Im-S &Object  & OOD Avg.  \\
\midrule 
LP & \multicolumn{1}{c|}{79.9} &69.8 &70.8 &46.4 &46.9 &52.1 &{57.2} \\ 
FT & \multicolumn{1}{c|}{81.3} &70.9 &65.6 &36.7 &46.3 &49.6 &{53.8}\\ 
L2-SP & \multicolumn{1}{c|}{81.7} &71.8 &70.0 &42.5 &48.5 &56.2 &57.8\\
LP-FT & \multicolumn{1}{c|}{81.7} &72.1 &73.5 &47.6 &50.3 &58.2 &60.3\\
FLYP & \multicolumn{1}{c|}{82.6} &73.0 &71.4 &48.1 &49.6 &58.7 &60.2\\
\hline
CLIP$\dag$ & \multicolumn{1}{c|}{68.3} &61.9 &\textbf{77.7} &49.9 &48.2 &54.2 & {58.4} \\
FLYP$^\dag$ & \multicolumn{1}{c|}{\textbf{82.6}} &\textbf{72.6} &71.8 &48.5 &49.8 &54.6 &59.5\\
Adapter$^\dag$ & \multicolumn{1}{c|}{{81.5}} &{71.7} &74.3 &50.3 &50.1 &55.3 &60.3\\
\rowcolor{tabhighlight}
\rowcolor{tabhighlight}
RAda-FT & \multicolumn{1}{c|}{81.4} &71.9 &75.5 &\textbf{51.7} &\textbf{50.4} &\textbf{56.8} & \textbf{61.3} \\
\bottomrule
\end{tabular}
}
\label{tab adapterfft}
\vspace{-0.5cm}
\end{table}

\section{Extending Other Fine-Tuning Ideas in FFT}
\label{sec morefftexp}
Our experiments demonstrate that RAda can be seamlessly integrated into the FFT setting by building on the existing practice~\cite{kumar2022fine}, effectively enhancing the baseline. This exploration has received limited attention in other alike fine-tuning approaches. To provide a comprehensive evaluation of our method, this section investigates extending the same concept to other fine-tuning approaches within the FFT framework.
Note that not all fine-tuning approaches are suitable for the FFT setting. For instance, CoOp~\cite{zhou2022learning} relies on the presence of a text encoder, which will be replaced by a linear classifier in FFT. Meanwhile, given additional inserted prompts will require large computational resources when updating all parameters, this section will focus exclusively on experiments extending an adapter-based fine-tuning approach that operates within the FFT framework without requiring a text encoder.
Specifically, we extend CLIP-Adapter~\cite{gao2024clip} by applying the same training strategy as LP-FT~\cite{kumar2022fine}. In this extension, we first train the feature adapter and then use the weights obtained after 5 epochs as initialization to fine-tune all learnable parameters. To ensure fair comparisons, we adopt the same settings as RAda for CLIP-Adapter, except for the learning rate, which is adjusted by factors of ${\times 0.1, \times 1, \times 10}$ relative to the original values in our implementation. The learning rate yielding the best performance on the evaluation sets in the ID data is selected for reporting results.

We list the experimental results in Table~\ref{tab adapterfft}. We observe that CLIP-Adapter can improve the OOD performance for FT when with the same sequential updating strategy, and it performs better than LP-FT in 3 out of the 5 OOD datasets evaluated. This is mainly because CLIP-Adapter can preserve the text features and part of the visual information, which is helpful for generalization~\cite{kumar2022fine}, as opposed to LP-FT, where the text information will be compromised for adaptation in the training data. 
However, as CLIP-Adapter cannot leverage fused representations from multiple modalities and inevitably distorts certain pretrained visual cues, it is inferior to RAda across all OOD datasets, with an average performance gap of 1pp. These results validate the effectiveness of our design of focusing on the final decision-making process during the fine-tuning process.

\section{Weight and Loss Format Ablation for Mask Regularization}
\label{sec weight}
%
To assess the sensitivity of RAda to variations in the regularization term $\mathcal{L}_{reg}$, we analyze its performance across different values of the hyperparameter $\alpha$ in the regularization loss, formulated as $\mathcal{L}'_{main} + \alpha * \mathcal{L}_{reg}$. 
The results, shown in Figure~\ref{fig alpha}, indicate that RAda remains robust to different values of $\alpha$ as long as it is within a reasonable range (\ie from 0.5 to 2.5), indicating that we do not need to specifically tune this hyper-parameter (the default $\alpha=1$ suffices). In practice, for all experiments in the EFT setting, it is fixed as 1.5 as it yields relatively better results. 

Moreover, since there are different alternatives for the loss terms, we also conduct experiments to analyze if the adopted $L_2$ regularization is the optimal choice by compring it with $L_1$ and $L_{\infty}$. Results in Table~\ref{tab reg} indicate that using the adopted L2 norm leads to better results than other alternatives.

\begin{figure}
    \centering
    \includegraphics[width=0.9\linewidth]{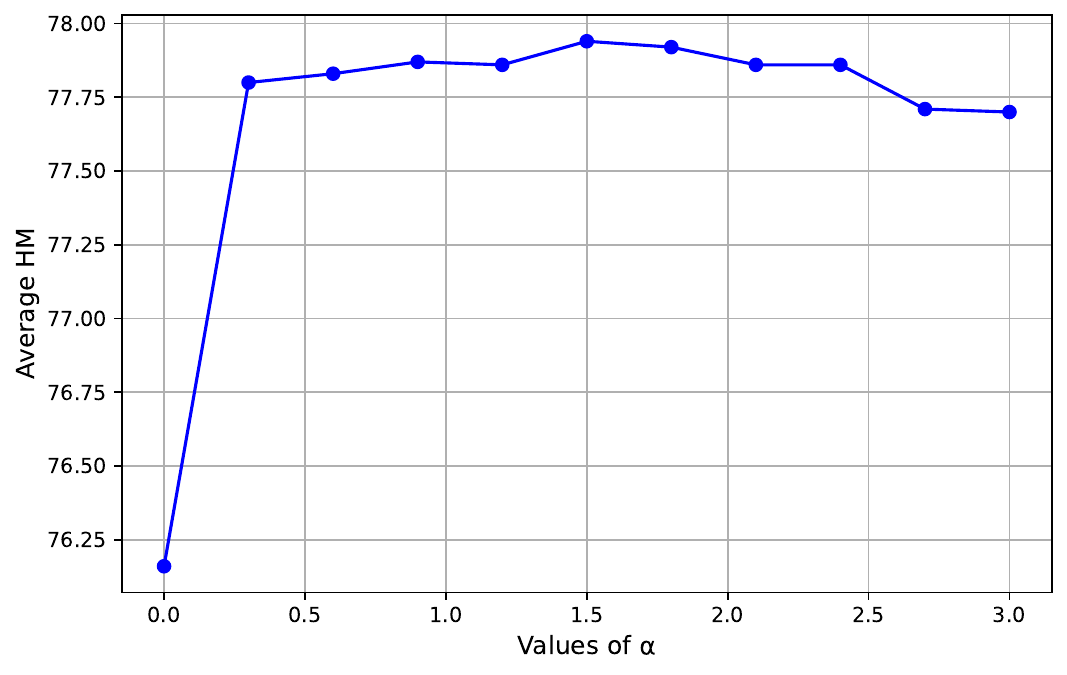}
    \vspace{-0.3cm}
    \caption{Evaluations of RAda with different weights (\ie $\alpha$) for the mask regularization term. Results are averaged over 11 datasets. RAda is insensitive regarding the different values of $\alpha$ as long as it is within a reasonable range (\ie from 0.5 to 2.5).}
    \label{fig alpha}
\end{figure}

\begin{table}[t]
    \centering
    \caption{Ablation on loss format for the mask regularization term. Results are averaged over 11 datasets.}
    \vspace{-0.3 cm}
    \begin{tabular}{lccc}
      \hline
      \textbf{loss format} &Base &Novel & \textbf{HM} \\
      \hline
      $L_{\infty}$ & 80.06 & 73.87 & 76.84 \\
      $L_1$        & 81.50 & 73.3 & 77.21 \\
      \rowcolor{tabhighlight}
      $L_2$        &82.16 &74.14 & 77.94 \\
      \hline
    \end{tabular}   
    \label{tab reg}
    \vspace{-0.5cm}
\end{table}

\section{Experiments with Different Backbones}
\label{sec backbones}
Our default setting in the manuscript employs the ViT-B/16 backbone for the CLIP image encoder. In this section, we investigate whether RAda can maintain its effectiveness with alternative backbones. Specifically, we evaluate RAda with another transformer-based image encoder (\ie ViT-B/32~\cite{dosovitskiy2020image}) and two ResNet-based variants of CLIP (\ie ResNet-50 and ResNet-101~\cite{he2016deep}). As presented in Table~\ref{tab backbone}, RAda demonstrates consistent improvements over the baseline across different backbones, achieving significant enhancements in base accuracy while exhibiting slight reductions in novel class performance.
In comparison with the established art CoOp~\cite{zhou2022learning}, RAda consistently outperforms it across all backbones in terms of harmonic mean between the base and novel accuracies, with particularly notable advantages in the unseen novel classes. These findings affirm the robustness and adaptability of RAda across varying backbone architectures.
\begin{table}[t]
    \centering
    \caption{Evaluations of RAda with different backbones for the CLIP image encoder. Results are averaged over 11 datasets.}
    \vspace{-0.3 cm}
    \scalebox{0.9}{
    \begin{tabular}{cccc}
    \toprule
         Backbone & Base Acc. &Novel Acc. & Hamonic Mean\\
         \midrule
         \multicolumn{4}{c}{ResNet-50} \\
         \hline
          CLIP & 65.29 & 69.01 & 67.09\\
          CoOp &76.56 &63.31 &69.31\\
          \rowcolor{tabhighlight}
            RAda &76.80 &67.30 &71.74 \\ 
        \midrule
         \multicolumn{4}{c}{ResNet-101} \\
         \hline
         CLIP & 64.53 & 69.82 & 67.07 \\
         CoOp &78.31 &63.80 &70.31\\
         \rowcolor{tabhighlight}
        RAda &78.42 &68.27 &72.99 \\     
         \midrule
          \multicolumn{4}{c}{ViT-B/32} \\
         \hline
          CLIP & 67.21 & 71.65 & 69.36 \\
          CoOp & 78.55 & 66.08 & 71.78 \\
          \rowcolor{tabhighlight}
        RAda & 76.31  & 70.45 & 73.26 \\
        \midrule
         \multicolumn{4}{c}{ViT-B/16} \\
         \hline
        CLIP & 69.34 & 74.22 & 71.70 \\
        CoOp & 82.69 &63.22 &71.66  \\
        \rowcolor{tabhighlight}
        RAda & 82.16 & 74.14 & 77.94 \\
         \bottomrule
    \end{tabular}}
    \label{tab backbone}
    \vspace{-0.5cm}
\end{table}

\section{Experiments with Different VLMs}
\label{sec vlms}
Following existing arts~\cite{zhou2022conditional,gao2024clip}, we conduct experiments only with CLIP in our manuscript.
But note that the rational matrix~\cite{chen2023domain} is applicable not only in CLIP’s similarity-based structure, but any cases when there is contrastive (CT) or softmax losses (SM) , as it represents the inner product's intermediate state when computing these losses. 
This extends RAda also in other VLMs, such as ALIGN or SigLIP where CL and SM are involved. 
To validate the effectiveness of RAda also in these different VLMs, we conduct experiments in the EFT setting and present the results in Table~\ref{tab vlm}. As seen, RAda is consistently effective even with different VLMs, indicating the broader applicability of RAda. 
\begin{table}[t]
    \centering
    \caption{Applying RAda in different VLMs. Results are averaged over 11 datasets.}
    \vspace{-0.3 cm}
    \begin{tabular}{lccc}
      \hline
      \textbf{VLM} & Base Acc.  & Novel Acc. & {Harmonic Mean} \\
      \hline
      \multicolumn{4}{c}{CLIP~\cite{radford2021learning}} \\
      \hline
      Zero Shot & 69.34 & 74.22 & 71.70 \\
      \rowcolor{tabhighlight}
      RAda      & 78.42 & 68.27 & 72.99 \\
      \hline
      \multicolumn{4}{c}{OpenCLIP~\cite{ilharco_gabriel_2021_5143773}} \\
      \hline
      Zero Shot & 67.61 & 71.08 & 69.30 \\
      \rowcolor{tabhighlight}
      RAda      & 78.00 & 73.31 & 75.58 \\
      \hline
      \multicolumn{4}{c}{SigLIP~\cite{zhai2023sigmoid}} \\
      \hline
      Zero Shot & 78.28 & 73.75 & 75.95 \\
      \rowcolor{tabhighlight}
      RAda      & 84.13 & 74.98 & 79.29 \\
      \hline
      \multicolumn{4}{c}{ALIGN~\cite{jia2021scaling}} \\
      \hline
      Zero Shot & 70.00 & 66.66 & 68.29 \\
      \rowcolor{tabhighlight}
      RAda      & 75.60 & 69.79 & 72.58 \\
      \bottomrule
    \end{tabular}   
    \label{tab vlm}
    \vspace{-0.2cm}
\end{table}

\section{More Attention Layers at the End}
\label{sec moreattn}
RAda attaches one additional attention layer at the end of CLIP for adaptation. To evaluate if more layers can better help the performance, we compare the original implementation with variants that use different attention layers. For the consecutive layers, we use the combined mask output from all previous layers as the query, and the rational matrix is still served as key and values for the new layers. The residual connection is utilized for the multi-layer implementation, where the final mask is the combination result of the masks obtained from all previous and current steps:
\begin{equation}
    \label{eq morelayers}
    \begin{aligned}
    \mathbf{M}^n = \mathbf{M}_0 + \mathbf{M}_1 +...+\mathbf{M}_n, \\
    \text{s.t.}~~\mathbf{M}_n = \mathcal{F}_m^n(\{{\mathbf{M}^{n-1}}, \mathbf{R}\}, \theta_m^n),
    \end{aligned}
\end{equation}
where $\mathbf{M}^n$ denotes using $n$ layers to obtain the final mask, $\mathbf{M}_n$ is the mask from the $n$-th layer, $\theta_m^n$ is the parameter for the $n$-th attention layer, and $\{{\mathbf{M}^{n-1}}, \mathbf{R}\}$ denotes the query is from $\mathbf{M}^{n-1}$, key and value are from $\mathbf{R}$ in the attention layer. Note $\mathbf{M}_0$ is implemented with $\mathbf{M}_0 = \mathcal{F}_m^0(\{\overline{\mathbf{h}}, \overline{\mathbf{f}}, \mathbf{R}\}, \theta_m^0)$ (\ie Eq. (6) in the manuscript) given there is no previous mask information.
Similarly, we report average results across 11 datasets from the base-to-new experiments of the EFT setting. As shown in Table~\ref{tab layers}, using more attention layers can improve the base performance but decrease the zero-shot ability of CLIP, indicating a trade-off between adapting to seen categories and maintaining the model's capacity to generalize to unseen categories. 
Since using only one layer for RAda can obtain similar results as that with two layers, we thus attach only one attention layer for simplicity in our implementation.

\begin{table}[t]
    \centering
    \caption{Evaluations of RAda with different attention layers attached at the end. Results are averaged over 11 datasets.}
    \vspace{-0.3 cm}
    \scalebox{0.9}{
    \begin{tabular}{lccc}
    \toprule
         & Base Acc. &Novel Acc. & Hamonic Mean\\
         \midrule
         Baseline & 69.34 &74.22 &71.70\\
         \hline
         \rowcolor{tabhighlight}
         1 layer &82.16 &74.14 &77.94\\
          2 layer &82.51 &73.84 &77.93\\
          3 layer &82.33 &71.81 &76.71\\
         \bottomrule
    \end{tabular}}
    \label{tab layers}
    \vspace{-0.3cm}
\end{table}

%

\end{document}